\documentclass{article}
\usepackage{iclr2027_conference,times}
\usepackage{amsmath,amssymb,amsthm,hyperref,graphicx,booktabs,xcolor,needspace,array,capt-of,placeins,colortbl,tabularx}
\usepackage[normalem]{ulem}
\usepackage{tikz,pgfplots}
\pgfplotsset{compat=1.18}
\definecolor{proofink}{HTML}{172B3B}
\definecolor{proofgrid}{HTML}{D9E0E4}
\definecolor{proofvalidfill}{HTML}{DCEBE2}
\definecolor{proofvalidedge}{HTML}{247348}
\definecolor{proofinvalidfill}{HTML}{F3E2E5}
\definecolor{proofblue}{HTML}{1469BD}
\definecolor{proofteal}{HTML}{298578}
\definecolor{proofgray}{HTML}{78828D}
\pgfplotsset{proofaxis/.style={
 axis line style={proofgray,thin},tick style={proofgray,thin},
 tick pos=left,axis lines=left,grid=major,
 grid style={proofgrid,line width=0.3pt},
 label style={font=\small,text=proofink},
 tick label style={font=\footnotesize,text=proofink},
 title style={font=\small\bfseries,text=proofink},
 every axis plot/.append style={line width=1.1pt},
 clip=false}}

\newcommand{\proofGroupSetupFigure}{%
\begin{figure}[!htb]
\centering\begingroup\color{proofink}
\begin{minipage}[t]{.46\linewidth}\centering\vspace{0pt}
\begin{tikzpicture}[x=3.0cm,y=3.0cm,font=\small,text=proofink]
\node[anchor=south,align=center,font=\small\bfseries] at (0.5,1.09) {(a) Target $q$};
\node[anchor=east] at (-0.05,0.833333) {\texttt{I}};
\filldraw[fill=proofvalidfill,draw=proofgrid,line width=0.45pt] (0.000000,0.666667) rectangle (0.500000,1.000000);
\node at (0.250000,0.833333) {0.40};
\filldraw[fill=proofinvalidfill,draw=proofgrid,line width=0.45pt] (0.500000,0.666667) rectangle (1.000000,1.000000);
\node at (0.750000,0.833333) {0};
\node[anchor=east] at (-0.05,0.500000) {\texttt{he}};
\filldraw[fill=proofinvalidfill,draw=proofgrid,line width=0.45pt] (0.000000,0.333333) rectangle (0.500000,0.666667);
\node at (0.250000,0.500000) {0};
\filldraw[fill=proofvalidfill,draw=proofgrid,line width=0.45pt] (0.500000,0.333333) rectangle (1.000000,0.666667);
\node at (0.750000,0.500000) {0.35};
\node[anchor=east] at (-0.05,0.166667) {\texttt{she}};
\filldraw[fill=proofinvalidfill,draw=proofgrid,line width=0.45pt] (0.000000,0.000000) rectangle (0.500000,0.333333);
\node at (0.250000,0.166667) {0};
\filldraw[fill=proofvalidfill,draw=proofgrid,line width=0.45pt] (0.500000,0.000000) rectangle (1.000000,0.333333);
\node at (0.750000,0.166667) {0.25};
\node[anchor=north] at (0.250000,-0.035) {\texttt{am}};
\node[anchor=north] at (0.750000,-0.035) {\texttt{is}};
\draw[proofvalidedge,line width=0.9pt] (0,0.666667) rectangle (0.500000,1);
\draw[proofvalidedge,line width=0.9pt] (0.500000,0) rectangle (1,0.666667);
\node[align=center,font=\small] at (0.5,-0.32) {$S_1=\{\texttt{I am}\}$\\$S_2=\{\texttt{he is},\texttt{she is}\}$};
\end{tikzpicture}
\end{minipage}
\hfill
\begin{minipage}[t]{.49\linewidth}\centering\vspace{0pt}
\begin{tikzpicture}[x=3.0cm,y=3.0cm,font=\small,text=proofink]
\node[anchor=south,align=center,font=\small\bfseries] at (0.5,1.09) {(b) CE predictor $p$};
\node[anchor=east] at (-0.05,0.800000) {\texttt{I}};
\filldraw[fill=proofvalidfill,draw=proofgrid,line width=0.45pt] (0.000000,0.600000) rectangle (0.400000,1.000000);
\node at (0.200000,0.800000) {0.16};
\filldraw[fill=proofinvalidfill,draw=proofgrid,line width=0.45pt] (0.400000,0.600000) rectangle (1.000000,1.000000);
\node at (0.700000,0.800000) {0.24};
\node[anchor=east] at (-0.05,0.425000) {\texttt{he}};
\filldraw[fill=proofinvalidfill,draw=proofgrid,line width=0.45pt] (0.000000,0.250000) rectangle (0.400000,0.600000);
\node at (0.200000,0.425000) {0.14};
\filldraw[fill=proofvalidfill,draw=proofgrid,line width=0.45pt] (0.400000,0.250000) rectangle (1.000000,0.600000);
\node at (0.700000,0.425000) {0.21};
\node[anchor=east] at (-0.05,0.125000) {\texttt{she}};
\filldraw[fill=proofinvalidfill,draw=proofgrid,line width=0.45pt] (0.000000,0.000000) rectangle (0.400000,0.250000);
\node at (0.200000,0.125000) {0.10};
\filldraw[fill=proofvalidfill,draw=proofgrid,line width=0.45pt] (0.400000,0.000000) rectangle (1.000000,0.250000);
\node at (0.700000,0.125000) {0.15};
\node[anchor=north] at (0.200000,-0.035) {\texttt{am}};
\node[anchor=north] at (0.700000,-0.035) {\texttt{is}};
\draw[proofvalidedge,line width=0.9pt] (0,0.600000) rectangle (0.400000,1);
\draw[proofvalidedge,line width=0.9pt] (0.400000,0) rectangle (1,0.600000);
\node[align=center,font=\small] at (0.5,-0.32) {$p(S_1)=0.16,\quad p(S_2)=0.36$};
\end{tikzpicture}
\end{minipage}
\endgroup
\caption{\presentationrevision{\textbf{Valid groups and a predictor that mixes them.} The target is the two-token example of Figure~\ref{fig:alpha-interpolation}. Rows are the first token and columns the second, as in that figure. Green cells are valid and red cells are invalid. In (b), row heights are $p_1$ and column widths are $p_2$, so cell areas equal joint probabilities. The two valid groups receive total probability $0.16+0.36=0.52$. The remaining $0.48$ falls on invalid combinations.}}
\label{fig:prop1-setting}
\end{figure}
}

\newcommand{\proofGroupWeightsFigure}{%
\begin{figure}[!htb]
\centering\begingroup\color{proofink}
\begin{minipage}[t]{.255\linewidth}\centering\vspace{0pt}
\begin{tikzpicture}[x=2.4cm,y=2.4cm,font=\small,text=proofink]
\node[anchor=south,align=center,font=\small\bfseries] at (0.5,1.09) {(a) Equal split};
\node[anchor=east] at (-0.05,0.750000) {\texttt{I}};
\filldraw[fill=proofvalidfill,draw=proofgrid,line width=0.45pt] (0.000000,0.500000) rectangle (0.500000,1.000000);
\filldraw[fill=proofinvalidfill,draw=proofgrid,line width=0.45pt] (0.500000,0.500000) rectangle (1.000000,1.000000);
\node[anchor=east] at (-0.05,0.375000) {\texttt{he}};
\filldraw[fill=proofinvalidfill,draw=proofgrid,line width=0.45pt] (0.000000,0.250000) rectangle (0.500000,0.500000);
\filldraw[fill=proofvalidfill,draw=proofgrid,line width=0.45pt] (0.500000,0.250000) rectangle (1.000000,0.500000);
\node[anchor=east] at (-0.05,0.125000) {\texttt{she}};
\filldraw[fill=proofinvalidfill,draw=proofgrid,line width=0.45pt] (0.000000,0.000000) rectangle (0.500000,0.250000);
\filldraw[fill=proofvalidfill,draw=proofgrid,line width=0.45pt] (0.500000,0.000000) rectangle (1.000000,0.250000);
\node[anchor=north] at (0.250000,-0.035) {\texttt{am}};
\node[anchor=north] at (0.750000,-0.035) {\texttt{is}};
\draw[proofvalidedge,line width=0.9pt] (0,0.500000) rectangle (0.500000,1);
\draw[proofvalidedge,line width=0.9pt] (0.500000,0) rectangle (1,0.500000);
\node[align=center,font=\small] at (0.5,-0.32) {$p(S_g)=0.25$};
\end{tikzpicture}
\end{minipage}
\hfill
\begin{minipage}[t]{.255\linewidth}\centering\vspace{0pt}
\begin{tikzpicture}[x=2.4cm,y=2.4cm,font=\small,text=proofink]
\node[anchor=south,align=center,font=\small\bfseries] at (0.5,1.09) {(b) Select group 2};
\node[anchor=east] at (-0.05,0.750000) {\texttt{he}};
\filldraw[fill=proofvalidfill,draw=proofgrid,line width=0.45pt] (0.000000,0.500000) rectangle (1.000000,1.000000);
\node[anchor=east] at (-0.05,0.250000) {\texttt{she}};
\filldraw[fill=proofvalidfill,draw=proofgrid,line width=0.45pt] (0.000000,0.000000) rectangle (1.000000,0.500000);
\node[anchor=north] at (0.500000,-0.035) {\texttt{is}};
\draw[proofvalidedge,line width=0.9pt] (0.000000,0) rectangle (1,1.000000);
\node[align=center,font=\small] at (0.5,-0.32) {$p(S_2)=1$};
\end{tikzpicture}
\end{minipage}
\hfill
\begin{minipage}[t]{.46\linewidth}\centering\vspace{0pt}
\begin{tikzpicture}
\begin{axis}[proofaxis,width=0.98\linewidth,height=4.4cm,
 xmin=0,xmax=1,ymin=0.6,ymax=1.42,
 xtick={0,0.5,1},ytick={0.75,1,1.25},
 xlabel={Probability $s$ of group 1},
 xlabel style={font=\footnotesize},ylabel={$W_\alpha(s)$},
 title={(c) Total group weight}]
 \addplot[proofgray,dashed,domain=0:1,samples=201] {x^0.6+(1-x)^0.6}
 node[pos=0.5,above,font=\footnotesize,text=proofink] {$\alpha=0.3$};
 \addplot[proofgray,domain=0:1] {1}
 node[pos=0.23,above,font=\footnotesize,text=proofink] {$\alpha=0.5$};
 \addplot[proofblue,domain=0:1,samples=201] {x^1.1+(1-x)^1.1}
 node[pos=0.5,below,font=\footnotesize,text=proofblue] {$\alpha=0.55$};
 \addplot[proofteal,domain=0:1,samples=201] {x^1.6+(1-x)^1.6}
 node[pos=0.5,above,font=\footnotesize,text=proofteal] {$\alpha=0.8$};
 \addplot[only marks,mark=*,mark size=1.6pt,proofink] coordinates {(0,1) (1,1)};
\end{axis}
\end{tikzpicture}
\end{minipage}
\endgroup
\caption{\presentationrevision{\textbf{Why the threshold is $1/m$.} For $m=2$, assigning group~1 probability $s$ at both positions gives $W_\alpha(s)=\sum_g p(S_g)^\alpha=s^{2\alpha}+(1-s)^{2\alpha}$. This is the total weight in Equation~\ref{eq:prop1-group-split}, not the objective itself. It is below one for $0<s<1$ when $\alpha>1/2$. At $\alpha=0.55$, (a) gives $W\approx0.933$, while (b) gives $W=1$.}}
\label{fig:prop1-weights}
\end{figure}
}

\newcommand{\proofGroupConditioningFigure}{%
\begin{figure}[!htb]
\centering\begingroup\color{proofink}
\begin{minipage}[t]{.31\linewidth}\centering\vspace{0pt}
\begin{tikzpicture}[x=2.8cm,y=2.8cm,font=\small,text=proofink]
\node[anchor=south,align=center,font=\small\bfseries] at (0.5,1.09) {(a) Original $p$};
\node[anchor=east] at (-0.05,0.800000) {\texttt{I}};
\filldraw[fill=proofvalidfill,draw=proofgrid,line width=0.45pt] (0.000000,0.600000) rectangle (0.400000,1.000000);
\filldraw[fill=proofinvalidfill,draw=proofgrid,line width=0.45pt] (0.400000,0.600000) rectangle (1.000000,1.000000);
\node[anchor=east] at (-0.05,0.425000) {\texttt{he}};
\filldraw[fill=proofinvalidfill,draw=proofgrid,line width=0.45pt] (0.000000,0.250000) rectangle (0.400000,0.600000);
\filldraw[fill=proofvalidfill,draw=proofgrid,line width=0.45pt] (0.400000,0.250000) rectangle (1.000000,0.600000);
\node[anchor=east] at (-0.05,0.125000) {\texttt{she}};
\filldraw[fill=proofinvalidfill,draw=proofgrid,line width=0.45pt] (0.000000,0.000000) rectangle (0.400000,0.250000);
\filldraw[fill=proofvalidfill,draw=proofgrid,line width=0.45pt] (0.400000,0.000000) rectangle (1.000000,0.250000);
\node[anchor=north] at (0.200000,-0.035) {\texttt{am}};
\node[anchor=north] at (0.700000,-0.035) {\texttt{is}};
\draw[proofvalidedge,line width=0.9pt] (0,0.600000) rectangle (0.400000,1);
\draw[proofvalidedge,line width=0.9pt] (0.400000,0) rectangle (1,0.600000);
\node[align=center,font=\small] at (0.5,-0.32) {$F_\alpha(p)\approx0.382$};
\end{tikzpicture}
\end{minipage}
\hfill
\begin{minipage}[t]{.31\linewidth}\centering\vspace{0pt}
\begin{tikzpicture}[x=2.8cm,y=2.8cm,font=\small,text=proofink]
\node[anchor=south,align=center,font=\small\bfseries] at (0.5,1.09) {(b) Condition on $S_1$};
\node[anchor=east] at (-0.05,0.500000) {\texttt{I}};
\filldraw[fill=proofvalidfill,draw=proofgrid,line width=0.45pt] (0.000000,0.000000) rectangle (1.000000,1.000000);
\node at (0.500000,0.500000) {1.00};
\node[anchor=north] at (0.500000,-0.035) {\texttt{am}};
\draw[proofvalidedge,line width=0.9pt] (0,0.000000) rectangle (1.000000,1);
\node[align=center,font=\small] at (0.5,-0.32) {$F_\alpha(p^{(1)})=0.400$};
\end{tikzpicture}
\end{minipage}
\hfill
\begin{minipage}[t]{.31\linewidth}\centering\vspace{0pt}
\begin{tikzpicture}[x=2.8cm,y=2.8cm,font=\small,text=proofink]
\node[anchor=south,align=center,font=\small\bfseries] at (0.5,1.09) {(c) Condition on $S_2$};
\node[anchor=east] at (-0.05,0.708333) {\texttt{he}};
\filldraw[fill=proofvalidfill,draw=proofgrid,line width=0.45pt] (0.000000,0.416667) rectangle (1.000000,1.000000);
\node at (0.500000,0.708333) {0.583};
\node[anchor=east] at (-0.05,0.208333) {\texttt{she}};
\filldraw[fill=proofvalidfill,draw=proofgrid,line width=0.45pt] (0.000000,-0.000000) rectangle (1.000000,0.416667);
\node at (0.500000,0.208333) {0.417};
\node[anchor=north] at (0.500000,-0.035) {\texttt{is}};
\draw[proofvalidedge,line width=0.9pt] (0.000000,0) rectangle (1,1.000000);
\node[align=center,font=\small] at (0.5,-0.32) {$F_\alpha(p^{(2)})\approx0.415$};
\end{tikzpicture}
\end{minipage}
\endgroup
\caption{\presentationrevision{\textbf{Conditioning improves the running predictor at $\alpha=0.55$.} Both conditioned predictors outperform (a). The comparison in Step~3 gives $F_\alpha(p)\approx0.382\le F_{\max}\sum_g p(S_g)^\alpha\approx0.388<F_{\max}\approx0.415$. Conditioning preserves the relative token probabilities within each group. It does not yet optimize them.}}
\label{fig:prop1-conditioning}
\end{figure}
}

\newcommand{\proofGroupCoverageFigure}{%
\begin{figure}[!htb]
\centering\begingroup\color{proofink}
\begin{minipage}[t]{.56\linewidth}\centering\vspace{0pt}
\begin{tikzpicture}
\begin{axis}[proofaxis,width=\linewidth,height=4.8cm,
 xmin=0,xmax=1,ymin=0.23,ymax=0.46,
 xtick={0,0.3213,1},xticklabels={0,0.321,1},ytick={0.25,0.35,0.40},
 xlabel={Probability $\varepsilon$ assigned to \texttt{she}},ylabel={$F_\alpha$},
 title={(a) Add a missing completion}]
 \addplot[proofblue,domain=0:1,samples=401] {0.35*(1-x)^0.55+0.25*x^0.55};
 \addplot[only marks,mark=*,mark size=1.7pt,proofink] coordinates {(0,0.35) (0.1,0.4007541361)};
 \addplot[only marks,mark=*,mark size=2pt,proofvalidedge] coordinates {(0.3213191589,0.4166952301)};
 \node[font=\footnotesize,anchor=north west,text=proofink] at (axis cs:0.02,0.345) {only \texttt{he is}};
 \node[font=\footnotesize,anchor=north west,text=proofink] at (axis cs:0.12,0.395) {$\varepsilon=0.1$};
 \node[font=\footnotesize,anchor=south west,text=proofvalidedge] at (axis cs:0.35,0.425) {best within $S_2$};
\end{axis}
\end{tikzpicture}
\end{minipage}
\hfill
\begin{minipage}[t]{.39\linewidth}\centering\vspace{0pt}
\begin{tikzpicture}[x=3.1cm,y=3.1cm,font=\small,text=proofink]
\node[anchor=south,align=center,font=\small\bfseries] at (0.5,1.09) {(b) Optimal predictor};
\node[anchor=east] at (-0.05,0.660660) {\texttt{he}};
\filldraw[fill=proofvalidfill,draw=proofgrid,line width=0.45pt] (0.000000,0.321319) rectangle (1.000000,1.000000);
\node at (0.500000,0.660660) {0.679};
\node[anchor=east] at (-0.05,0.160660) {\texttt{she}};
\filldraw[fill=proofvalidfill,draw=proofgrid,line width=0.45pt] (0.000000,-0.000000) rectangle (1.000000,0.321319);
\node at (0.500000,0.160660) {0.321};
\node[anchor=north] at (0.500000,-0.035) {\texttt{is}};
\draw[proofvalidedge,line width=0.9pt] (0.000000,0) rectangle (1,1.000000);
\node[align=center,font=\small] at (0.5,-0.32) {$F_\alpha\approx0.417>0.400$};
\end{tikzpicture}
\end{minipage}
\endgroup
\caption{\presentationrevision{\textbf{Selecting a group preserves all of its completions.} At $\alpha=0.55$, start with a predictor that always generates \texttt{he is} and give \texttt{she} probability $\varepsilon$ at the first position. The objective becomes $0.35(1-\varepsilon)^{0.55}+0.25\varepsilon^{0.55}$, which increases for small positive $\varepsilon$. Its maximum assigns positive probability to both valid completions. This best predictor within $S_2$ also outperforms the only predictor within $S_1$.}}
\label{fig:prop1-coverage}
\end{figure}
}

\hypersetup{hidelinks}
\definecolor{targetblue}{HTML}{005EA8}
\definecolor{stateochre}{HTML}{A65E00}
\DeclareRobustCommand{\cleanvar}[1]{{\color{targetblue}#1}}
\DeclareRobustCommand{\statevar}[1]{{\color{stateochre}#1}}
\usepackage{mdframed}
\newtheoremstyle{paperproposition}{0pt}{0pt}{\normalfont}{}{\bfseries}{.}{.5em}{}
\theoremstyle{paperproposition}
\newtheorem{proposition}{Proposition}
\definecolor{statementrule}{HTML}{245F9E}
\definecolor{statementfill}{HTML}{F5F8FB}
\definecolor{referencefill}{HTML}{EDF1F5}

\mdfdefinestyle{propositionstyle}{
 hidealllines=true,leftline=true,linewidth=0.9pt,
 linecolor=statementrule,backgroundcolor=statementfill,
 innerleftmargin=7pt,innerrightmargin=7pt,
 innertopmargin=5pt,innerbottommargin=1pt,
 skipabove=7pt,skipbelow=7pt,roundcorner=0pt,nobreak=true
}
\surroundwithmdframed[style=propositionstyle]{proposition}

\DeclareRobustCommand{\keyref}[3][]{\hyperref[#3]{#2~\ref*{#3}#1}}
\newcommand{\keyequation}[1]{%
 \begingroup
 \setlength{\fboxsep}{7pt}%
 \colorbox{referencefill}{$\displaystyle #1$}%
 \endgroup
}

\newmdenv[
 hidealllines=true,leftline=true,linewidth=1.3pt,
 linecolor=statementrule,backgroundcolor=statementfill,
 innerleftmargin=8pt,innerrightmargin=8pt,
 innertopmargin=6pt,innerbottommargin=6pt,
 skipabove=7pt,skipbelow=7pt,nobreak=true
]{methodtakeaway}

\newmdenv[
 hidealllines=true,leftline=true,linewidth=1.3pt,
 linecolor=statementrule,backgroundcolor=statementfill,
 innerleftmargin=8pt,innerrightmargin=8pt,
 innertopmargin=4pt,innerbottommargin=4pt,
 skipabove=6pt,skipbelow=5pt,nobreak=true
]{experimenttakeaway}

\title{Alpha Diffusion Language Models:\\
Factorization Alone Is Not the Problem}
\author{Nikita Gushchin$^{1,2}$ \quad Dmitry Baranchuk$^{3}$ \quad Alexander Korotin$^{1,2}$\\
\normalfont $^{1}$Applied AI Institute, Moscow, Russia \quad $^{2}$AXXX, Russia\\
\normalfont $^{3}$Yandex Research}
\iclrfinalcopy
\hypersetup{pdftitle={Alpha Diffusion Language Models: Factorization Alone Is Not the Problem},
 pdfauthor={Nikita Gushchin, Dmitry Baranchuk, Alexander Korotin}}
\newcommand{\revision}[1]{#1}
\DeclareRobustCommand{\newrevision}[1]{#1}
\newif\ifshowrevisions
\showrevisionsfalse
\ifshowrevisions
  \colorlet{revisiontext}{red}
\else
  \colorlet{revisiontext}{black}
\fi
\DeclareRobustCommand{\presentationrevision}[1]{\textcolor{revisiontext}{#1}}
\DeclareRobustCommand{\reviewrevision}[1]{#1}
\begin{document}
\maketitle
\fancyhead{}
\renewcommand{\headrulewidth}{0pt}

\begin{abstract}
Discrete diffusion language models can generate multiple tokens in parallel, but reducing the number of denoising steps can lead to inconsistent predictions. Standard cross-entropy training fits conditional token marginals, whereas parallel generation requires consistent joint predictions. We introduce \emph{Alpha Diffusion Language Models (AlphaDLM)}, trained with a sequence-level alpha loss that recovers cross-entropy in the limit of vanishing alpha and has a joint-mode optimum at alpha one. Our analysis characterizes how the objective and factorization jointly determine the fitted distribution. We identify conditions under which intermediate alpha preserves multiple valid completions while excluding invalid token combinations. Trained on TinyGSM, \newrevision{our method} achieves \revision{34.6\%} accuracy on GSM8K with only four model evaluations. We further scale the method to \revision{SDAR-1.7B} and evaluate it on code and mathematics benchmarks. These results show that changing the training objective can improve the accuracy-computation trade-off of factorized diffusion language models.
\end{abstract}

\begin{figure}[!ht]
\centering
\includegraphics[width=\linewidth]{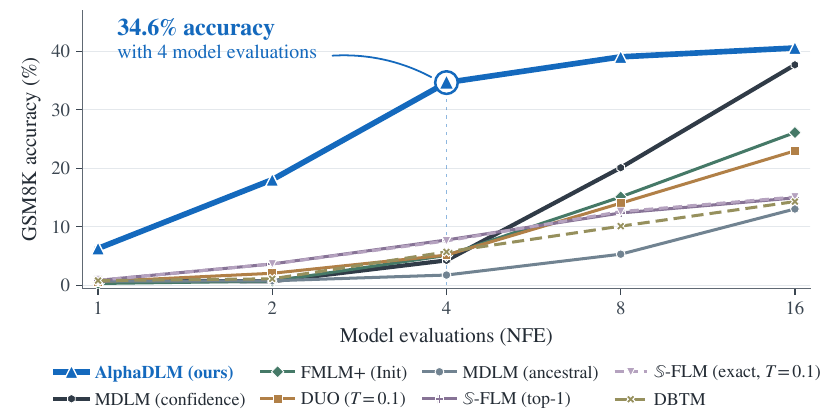}
\caption{\textbf{GSM8K accuracy with 1--16 model evaluations.} A single AlphaDLM \newrevision{(ours)} checkpoint \presentationrevision{with $k=16$} reaches 34.6\% in four evaluations. \reviewrevision{AlphaDLM and MDLM (confidence) use \presentationrevision{fixed-NFE confidence sampling} with the token-reveal schedule of FMLM+.} Other methods \newrevision{use} their respective samplers. \reviewrevision{\uline{Sources and protocols}: Table~\ref{tab:tinygsm-published} and Appendix~\ref{app:tinygsm-protocols}.}}
\label{fig:tinygsm-teaser}
\end{figure}

\section{Introduction}
Discrete diffusion language models generate text by reconstructing corrupted sequences \citep{sahoo2024mdlm,nie2025llada}. A single evaluation predicts distributions at many positions, allowing several tokens to be committed in parallel. \reviewrevision{Independently predicted tokens can each fit the visible context yet form an invalid sequence when generated together \citep{wu2025fastdllm}.}

\reviewrevision{Under cross-entropy training, the optimal factorized denoiser fits the conditional distribution of each token separately. When several completions are possible, independently combining these predictions can produce invalid sequences (\keyref{Figure}{fig:marginals-vs-modes}). One approach is to model token dependencies explicitly \citep{liu2025copula,xu2025energy}. However, a factorized distribution can already represent any single valid completion by assigning probability one to its tokens. Factorization therefore does not by itself prevent consistent parallel predictions. This suggests changing the training objective to favor token distributions that produce valid completions when combined.}

\reviewrevision{We introduce \emph{Alpha Diffusion Language Models (AlphaDLM)}, trained with sequence-level alpha loss, a form of generalized cross-entropy \citep{zhang2018gce}. This loss includes cross-entropy as the limit $\alpha\to0$, while at $\alpha=1$, predicting a joint mode minimizes the expected loss. These endpoints suggest a way to move from fitting token marginals toward fitting complete sequences. We analyze what happens between them when the denoiser is factorized. The optimum can preserve several valid completions while excluding invalid combinations, rather than selecting only one answer. \newrevision{We establish sufficient conditions for this behavior.}}

\reviewrevision{This change in the training objective is intended to improve generation when many tokens must be predicted in parallel within a small number of model evaluations. Recent approaches pursue this goal through trajectory supervision or self-distillation \citep{zhang2026t3d,chen2026dparallel,kim2026cdlm}. \newrevision{Our method} instead trains directly on corrupted data, without teacher-generated trajectories, and \newrevision{uses} the factorized denoiser. We evaluate whether this change improves accuracy at a fixed generation budget, using TinyGSM for comparisons between training objectives and SDAR-1.7B for experiments on code and mathematics. \newrevision{Figure~\ref{fig:tinygsm-teaser} summarizes the low-NFE result on GSM8K.} Appendix~\ref{app:related} \newrevision{further} discusses \uline{related work}.}

\begin{minipage}{\linewidth}
\paragraph{Contributions.}
We combine this training approach with an analysis of its factorized optimum and experiments on parallel generation. Our main contributions are:
\par\vspace{4pt}
\setlength{\topsep}{7pt}
\setlength{\itemsep}{5pt}
\setlength{\parsep}{0pt}
\begin{enumerate}
\item \reviewrevision{We introduce \emph{Alpha Diffusion Language Models}: diffusion language models whose factorized denoisers are trained with sequence-level alpha loss. We characterize the distribution that minimizes this loss within the class of factorized models and explain how it differs from the optimum of token-wise alpha loss (Sections~\ref{sec:alpha-loss}--\ref{sec:alpha-training}).}
\item \reviewrevision{We establish conditions under which the optimal factorized denoiser trained with sequence-level alpha loss assigns probability only to valid completions and can preserve several possible answers (Section~\ref{sec:mode-selection}, Appendix~\ref{app:compatible-support}).}
\item \reviewrevision{We empirically demonstrate that our method achieves 34.6\% GSM8K accuracy with four model evaluations and 7.66\% with a single evaluation after training on TinyGSM. Sequence-level training also achieves higher accuracy than token-wise alpha loss at matched evaluation budgets (Section~\ref{sec:tinygsm}). We also demonstrate improved accuracy--TPF trade-offs with SDAR-1.7B on code and mathematics benchmarks (Section~\ref{sec:sdar-math}).}
\end{enumerate}
\end{minipage}

\begin{figure}[!htb]
    \centering
    \includegraphics[width=\linewidth]{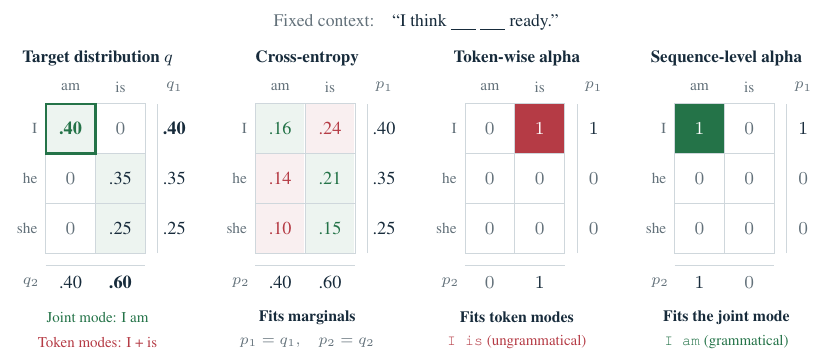}
    \caption{\revision{\textbf{Marginal fitting, token-mode fitting, and joint-mode fitting.}
    Cross-entropy assigns probability $0.48$ to ungrammatical completions.
    At $\alpha=1$, token-wise alpha loss predicts \texttt{I is}, while sequence-level alpha loss fits the joint mode \texttt{I am}.
    Each fitted distribution minimizes its population loss within the same factorized family.
    Green and red mark grammatical and ungrammatical outcomes with positive probability.}}
    \label{fig:marginals-vs-modes}
\end{figure}

\Needspace{7\baselineskip}
\section{Background}
\reviewrevision{To understand how the training objective affects parallel generation, we first review how discrete diffusion language models are trained and decoded (Section~\ref{sec:background-diffusion}). We then turn to generalized cross-entropy, a family of losses that includes standard cross-entropy as a limiting case (Section~\ref{sec:background-objectives}).}

\subsection{Discrete diffusion language models}
\label{sec:background-diffusion}
\reviewrevision{Discrete diffusion models learn to reverse a corruption process \citep{austin2021structured,lou2024sedd}. A denoiser predicts clean tokens from corrupted inputs, and a sampler uses these predictions to reconstruct the sequence.}

\Needspace{10\baselineskip}
\paragraph{Corruption process.}
\reviewrevision{Let $\cleanvar{x}\in\mathcal V^L$ be a clean sequence conditioned on a prompt or preceding blocks $c$, and let $\statevar{z}\sim q_t(\cdot\mid\cleanvar{x})$ be its corruption at time $t\in[0,1]$. Discrete corruption replaces tokens with vocabulary symbols or a mask. Masked diffusion independently \newrevision{keeps} each token or replaces it with $\mathtt m\notin\mathcal V$:}
\begin{equation}
 q_t(\statevar{z}_i\mid \cleanvar{x}_i)=\bar\alpha_t\mathbf{1}\{\statevar{z}_i=\cleanvar{x}_i\}+(1-\bar\alpha_t)\mathbf{1}\{\statevar{z}_i=\mathtt m\},
 \qquad q_t(\statevar{z}\mid \cleanvar{x})=\prod_{i=1}^L q_t(\statevar{z}_i\mid \cleanvar{x}_i).
\end{equation}
\newrevision{The probability of keeping a token unchanged decreases} from $\bar\alpha_0=1$ to $\bar\alpha_1=0$. Despite independent corruption, the clean tokens can remain dependent conditional on $\statevar{z}$.

\paragraph{Factorized denoiser.}
\newrevision{For this masking process, the optimal clean-token denoiser under cross-entropy training is independent of time $t$ \citep[Proposition~3.2]{zheng2025timeagnostic}. We therefore write the denoiser without explicit time conditioning.}
\newrevision{Let $M=\{i:\statevar{z}_i=\mathtt m\}$ be the masked positions and $\cleanvar{x}_M=(\cleanvar{x}_i)_{i\in M}$ their original tokens, in sequence order. Given the visible context, the denoiser predicts a categorical distribution at each position in $M$:}
\begin{equation}
 p_\theta(\cleanvar{x}_M\mid \statevar{z},c)=\prod_{i\in M}p_{\theta,i}(\cleanvar{x}_i\mid \statevar{z},c).
\end{equation}
\reviewrevision{Factorization holds within each evaluation. Later predictions depend on previously revealed tokens, so the generated sequence can contain dependencies despite the factorized denoiser.}

\paragraph{Training objective.}
\reviewrevision{The denoiser learns from corrupted inputs using weighted token cross-entropy \citep{sahoo2024mdlm,shi2024md4}:}
\begin{equation}
 \mathcal L_{\mathrm{CE}}(\theta)=
 \mathbb E_{\cleanvar{x},c,t,\statevar{z}}\left[-w(t)\sum_{i\in M}\log p_{\theta,i}(\cleanvar{x}_i\mid \statevar{z},c)\right],
\end{equation}
\reviewrevision{The expectation is over data, training times, and corruptions $\statevar{z}\sim q_t(\cdot\mid\cleanvar{x})$. With uniform times and $\bar\alpha_t=1-t$, the continuous-time masked diffusion bound uses $w(t)=1/t$.}

For a token value $v\in\mathcal V$, fixed $(\statevar{z},c)$, and unrestricted token predictors, the population cross-entropy is minimized by the conditional token marginals, $p_i^*(v\mid \statevar{z},c)=p_{\mathrm{data}}(\cleanvar{x}_i=v\mid \statevar{z},c)$.
Their product can differ from the joint posterior over the missing tokens. In particular, independently selecting their modes can produce invalid token combinations.

\paragraph{Parallel decoding.}
\label{sec:parallel-decoding}
\presentationrevision{Confidence sampling selects positions to reveal using the maximum predicted token probability at each position. In our confidence-sampling experiments, revealed tokens are chosen by argmax. \emph{Fixed-NFE confidence sampling} uses a prescribed evaluation budget and ranks positions by confidence to meet a reveal schedule~\citep{nie2025llada}. The fixed-NFE sampler used for TinyGSM also reveals positions above a threshold (Section~\ref{sec:tinygsm}). \emph{Adaptive confidence sampling} follows Fast-dLLM~\citep{wu2025fastdllm}: it reveals positions above a threshold, or the most confident position if none qualifies. It runs until all positions are filled, so NFE varies across examples.}

\paragraph{Blockwise diffusion.}
\reviewrevision{The same denoising procedure can operate within blocks generated autoregressively \citep{arriola2025block,cheng2025sdar}. Writing $\cleanvar{x}^{(b)}$ for block $b$ among $B$ blocks and $\cleanvar{x}^{(<b)}$ for all preceding blocks, the full sampling distribution is}
\begin{equation}
 \reviewrevision{p_\theta^{\mathrm{gen}}(\cleanvar{x}\mid c)=\prod_{b=1}^B p_\theta^{\mathrm{gen}}(\cleanvar{x}^{(b)}\mid \cleanvar{x}^{(<b)},c).}
\end{equation}
\reviewrevision{Each block uses iterative calls to $p_\theta$ with the preceding context cached.}

\Needspace{6\baselineskip}
\subsection{Generalized cross-entropy}
\label{sec:background-objectives}
\newrevision{The cross-entropy used to train the diffusion denoiser is the standard loss for categorical prediction. For a target class $y$ and predicted class probabilities $p$, it is $\ell_{\mathrm{CE}}(p,y)=-\log p(y)$. This familiar loss is a limiting case of a broader family, generalized cross-entropy, also known as alpha loss \citep{zhang2018gce}:}
\begin{equation}
 \reviewrevision{\ell_\alpha(p,y)=\frac{1-p(y)^\alpha}{\alpha},
 \qquad \lim_{\alpha\to0}\ell_\alpha(p,y)=-\log p(y).}
 \label{eq:background-alpha-loss}
\end{equation}
\reviewrevision{The parameter $\alpha$ controls the penalty for low target probabilities. Cross-entropy diverges as $p(y)\to0$, whereas alpha loss is bounded by $1/\alpha$ for $\alpha>0$. At $\alpha=1$, it reduces to $1-p(y)$.}

\section{Alpha Diffusion Language Models}
{
\newrevision{We first define our sequence-level objective by applying alpha loss to the joint probability of the masked tokens (Section~\ref{sec:alpha-loss}). We then analyze how factorization changes the distribution that minimizes this loss, compare sequence-level and token-wise training, and establish conditions for excluding invalid token combinations (Section~\ref{sec:mode-selection}). Finally, we describe how we normalize the loss across different mask counts and train the denoiser for use with parallel and blockwise samplers (Section~\ref{sec:alpha-training}).}

\subsection{Sequence-level denoising objective}
\label{sec:alpha-loss}
\reviewrevision{We seek a training objective that improves joint predictions while \newrevision{using} a factorized denoiser. Given a corrupted sequence $\statevar{z}$, the denoiser reconstructs the original tokens at the masked positions $M$. \newrevision{As in Section~\ref{sec:background-diffusion}, we omit explicit time conditioning because the conditional distribution of the original tokens given $\statevar{z}$ and $c$ does not depend on $t$.} The factorized denoiser assigns them the joint probability}
\[
 \reviewrevision{p_\theta(\cleanvar{x}_M\mid\statevar{z},c)=\prod_{i\in M}p_{\theta,i}(\cleanvar{x}_i\mid\statevar{z},c).}
\]
\reviewrevision{Applying cross-entropy to this probability still gives a sum of token-level losses, because the logarithm turns the product into a sum. For positive $\alpha$, generalized cross-entropy (Section~\ref{sec:background-objectives}) applied to this product no longer gives a sum of separate token losses. This yields the sequence-level alpha loss:}
\begin{equation}
\keyequation{
 \ell^{\mathrm{seq}}_\alpha(\theta;\cleanvar{x}_M,\statevar{z},c)
 =\newrevision{\frac{1-p_\theta(\cleanvar{x}_M\mid\statevar{z},c)^\alpha}{\alpha}}
 =\frac{1-\left[\prod_{i\in M}p_{\theta,i}(\cleanvar{x}_i\mid\statevar{z},c)\right]^\alpha}{\alpha}.
}
 \label{eq:matched-alpha-objectives}
\end{equation}
\nopagebreak[4]
\reviewrevision{This loss recovers the sum of token cross-entropies as $\alpha\to0$~\newrevision{(Section~\ref{sec:background-objectives})}.}

\subsection{From marginal fitting to joint mode fitting}
\label{sec:mode-selection}
\begingroup
\setlength{\abovedisplayskip}{5pt plus 1pt minus 1pt}
\setlength{\belowdisplayskip}{5pt plus 1pt minus 1pt}
\setlength{\abovedisplayshortskip}{3pt plus 1pt}
\setlength{\belowdisplayshortskip}{4pt plus 1pt minus 1pt}
\reviewrevision{Fix a corrupted sequence $\statevar{z}$ and context $c$. Let $q$ denote the conditional distribution of the original tokens at the masked positions:}
\begin{equation}
 \reviewrevision{q(\cleanvar{y})=p_{\mathrm{data}}(\cleanvar{x}_M=\cleanvar{y}\mid\statevar{z},c).}
 \label{eq:conditional-target}
\end{equation}
\reviewrevision{We compare the distribution minimizing the expected sequence-level alpha loss with and without factorization over the $m=|M|$ masked positions, indexed by $i\in M$. \newrevision{Figure~\ref{fig:alpha-interpolation} compares these two cases with token-wise training at the same values of $\alpha$.}}

\noindent\newrevision{\textbf{Without factorization.} We first allow the predictor to represent any joint distribution. For $0<\alpha<1$, we define the lower-temperature target}
\begin{equation}
 \keyequation{\newrevision{
 q_T(\cleanvar{y}):=\frac{q(\cleanvar{y})^{1/T}}{\sum_{\cleanvar{y}'}q(\cleanvar{y}')^{1/T}},
 \qquad T=1-\alpha.
 }}
 \label{eq:tempered-target}
\end{equation}
\newrevision{The temperature $T$ changes the target distribution, not the sampling procedure. The expected sequence-level alpha loss is minimized by fitting this target \citep[Proposition~1]{sypherd2022tunable}:}
\begin{equation}
 \newrevision{p_\alpha^{*,\mathrm{joint}}(\cleanvar{y})=q_T(\cleanvar{y}).}
 \label{eq:joint-temperature-optimum}
\end{equation}
\reviewrevision{At $\alpha=1$, any distribution supported on joint modes is optimal. For $\alpha>1$, only point masses on joint modes are optimal. Appendix~\ref{app:population-optima} gives \uline{the parameter correspondence and the elementary argument for $\alpha\ge1$}.}

\begin{methodtakeaway}[innertopmargin=4pt,innerbottommargin=4pt,skipabove=5pt,skipbelow=5pt]
\newrevision{\textbf{Without factorization, alpha loss fits a lower-temperature target.} For $0<\alpha<1$, more likely complete sequences receive a larger share of the probability (Figure~\ref{fig:alpha-interpolation}, top row).}
\end{methodtakeaway}

\reviewrevision{\textbf{With factorization.} Lowering the temperature does not remove token dependencies, so $q_T$ generally cannot be represented by a factorized denoiser. The optimal denoiser instead fits the marginals of a joint distribution that balances closeness to $q_T$ against dependence between tokens. For $0<\alpha<1$, this distribution solves \citep[Lemma~8, extended to $m$ factors]{lapidoth2019dependence}:}
\begin{equation}
 \pi_\alpha^*\in\operatorname*{arg\,min}_{\pi}
 \left\{D_{\mathrm{KL}}(\pi\Vert q_T)
 +\frac{\alpha}{1-\alpha}\operatorname{TC}(\pi)\right\},
 \quad
 p_\alpha^{*,\mathrm{fact}}=\prod_{i\in M}\pi_{\alpha,i}^*.
 \label{eq:factorized-target-tradeoff}
\end{equation}

\reviewrevision{The KL term measures the departure from $q_T$. Total correlation, $\operatorname{TC}(\pi)=D_{\mathrm{KL}}(\pi\Vert\prod_{i\in M}\pi_i)$, measures dependence between tokens and is zero when they are independent. Appendix~\ref{app:population-optima} gives \uline{the derivation, including zero probabilities}.}

\begin{methodtakeaway}[innertopmargin=4pt,innerbottommargin=4pt,skipabove=5pt,skipbelow=5pt]
\newrevision{\textbf{With factorization, the denoiser fits the marginals of $\pi_\alpha^*$.} The auxiliary distribution $\pi_\alpha^*$ balances closeness to $q_T$ against token dependence. The denoiser represents the product of its marginals (Figure~\ref{fig:alpha-interpolation}, middle row).}
\end{methodtakeaway}

\reviewrevision{\textbf{Why sequence-level rather than token-wise?} Applying alpha loss separately to each token probability gives}
\begin{equation}
 \reviewrevision{\ell^{\mathrm{tok}}_\alpha=\sum_{i\in M}\frac{1-p_{\theta,i}(\cleanvar{x}_i\mid\statevar{z},c)^\alpha}{\alpha}.}
 \label{eq:tokenwise-alpha-objective}
\end{equation}
\reviewrevision{For $0<\alpha<1$, minimizing its expectation gives}
\begin{equation}
 p_{\alpha,i}^{*,\mathrm{tok}}(\cleanvar{y}_i)
 =\frac{q_i(\cleanvar{y}_i)^{1/T}}{\sum_{\cleanvar{y}_i^{\prime}\in\mathcal V}q_i(\cleanvar{y}_i^{\prime})^{1/T}},
 \qquad T=1-\alpha.
 \label{eq:tokenwise-temperature-optimum}
\end{equation}
\reviewrevision{Here $q_i$ is the marginal of $q$ at position $i$, and $\cleanvar{y}_i\in\mathcal V$ is a token value. When $q$ factorizes, the two objectives have the same optimum.}

\begin{methodtakeaway}[innertopmargin=4pt,innerbottommargin=4pt,skipabove=5pt,skipbelow=5pt]
\newrevision{\textbf{Sequence-level fitting uses the joint target.} Token-wise loss sharpens each marginal of $q$ separately. Sequence-level loss instead fits the marginals of $\pi_\alpha^*$, which depends on the joint distribution $q$. Figures~\ref{fig:marginals-vs-modes} and~\ref{fig:alpha-interpolation} compare the endpoints and intermediate~$\alpha$, respectively.}
\end{methodtakeaway}

\endgroup

\par\medskip\noindent\begin{minipage}{\linewidth}
    \centering
    \includegraphics[width=0.90\linewidth]{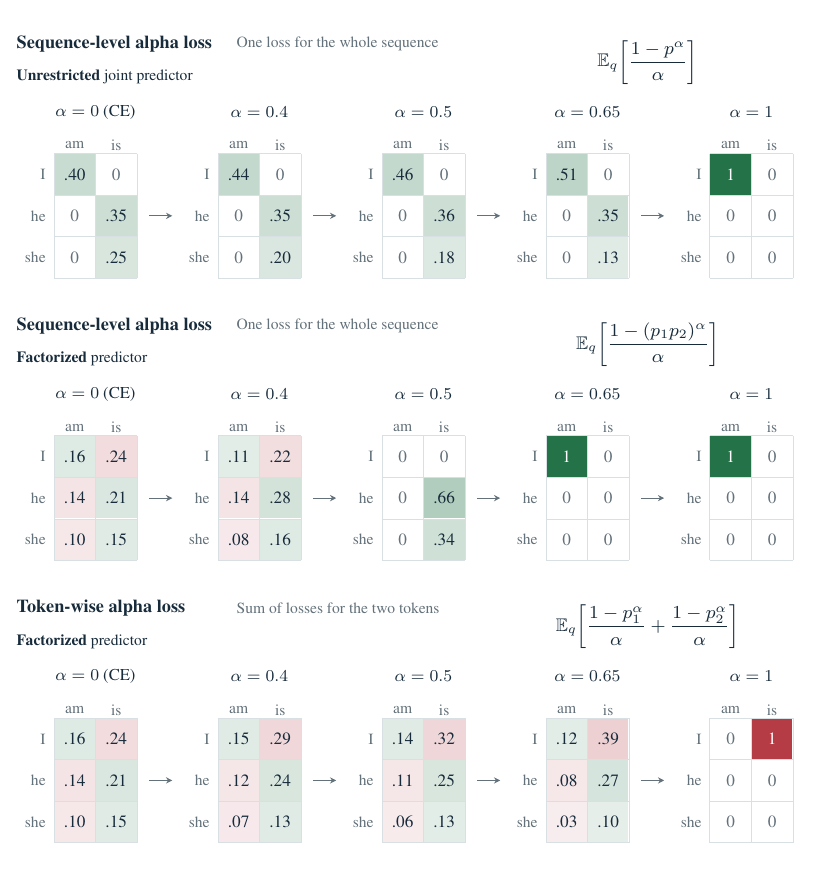}
    \captionof{figure}{\textbf{Interpolating from cross-entropy to mode fitting.}
    \newrevision{Rows show globally optimal distributions for Figure~\ref{fig:marginals-vs-modes}'s target: sequence-level loss without factorization, sequence-level loss with factorization, and token-wise loss with factorization. Columns use the same $\alpha$, with $\alpha=0$ denoting cross-entropy. At $\alpha=0.5$, the unrestricted optimum $q_T$ supports all three original completions, while the sequence-level factorized optimum supports only \texttt{he is} and \texttt{she is}. Both sequence-level rows reach \texttt{I am} at $\alpha=1$, whereas token-wise fitting reaches \texttt{I is}. Here $p=p(x_1,x_2)$ and $p_i=p_i(x_i)$. Green marks grammatical combinations, red ungrammatical ones. Entries are rounded.}}
    \label{fig:alpha-interpolation}
\end{minipage}\par\medskip

\newrevision{\textbf{Preserving multiple valid completions.} Can a factorized predictor exclude invalid token combinations while assigning positive probability to several completions? Here, valid completions are those with positive probability under $q$. The following result gives sufficient conditions for this behavior.}
\begin{proposition}[Multiple valid completions]
\label{prop:compatible-mixtures}
\reviewrevision{\newrevision{Consider $m\ge2$ masked tokens, and suppose the valid completions split into groups} such that any within-group combination of tokens is valid, and different groups use disjoint token sets at every position. For $1/m<\alpha<1$, every globally optimal factorized predictor selects one group and assigns positive probability to all completions in that group.}
\end{proposition}

\newrevision{To illustrate the behavior described by Proposition~\ref{prop:compatible-mixtures}, we consider the two-token example ($m=2$) in Figure~\ref{fig:alpha-interpolation} at $\alpha=1/m=1/2$, the boundary of the proposition's range. Without factorization, the optimal predictor still assigns positive probability to all three original completions. With factorization, the first position predicts either \texttt{he} or \texttt{she}, and the second always predicts \texttt{is}. Independent predictions then produce only \texttt{he is} and \texttt{she is}, both valid completions. This is a global optimum for this example, as verified separately at the boundary. Appendix~\ref{app:compatible-support} gives \uline{the calculation for this example}. As $\alpha$ increases further, the sequence-level optimum switches to \texttt{I am}, whereas token-wise fitting approaches the invalid combination \texttt{I is}.}

\presentationrevision{Appendix~\ref{app:compatible-support} gives \uline{an illustrated elementary proof} and a counterexample showing why the support condition matters.}

\Needspace{8\baselineskip}
\subsection{Training and decoding}
\label{sec:alpha-training}
\newrevision{During training, examples have different numbers of masked tokens $m=|M|$. Multiplying more token probabilities makes the joint probability smaller, even when the probability assigned to each target token stays the same. We therefore normalize the log-probability by $m$. For $m>0$, let $s_\theta=\frac{1}{m}\sum_{i\in M}\log p_{\theta,i}(\cleanvar{x}_i\mid\statevar{z},c)$. With a fixed training parameter $k$, we use:}
\begin{equation}
 \ell_k(s_\theta)=\frac{1-\exp(k s_\theta)}{k}
 =\frac{1-p_\theta(\cleanvar{x}_M\mid \statevar{z},c)^{\alpha_{\mathrm{eff}}}}
 {m\alpha_{\mathrm{eff}}},
 \qquad \alpha_{\mathrm{eff}}=\frac{k}{m}.
 \label{eq:effective-alpha}
\end{equation}
\newrevision{This is sequence-level alpha loss (Equation~\ref{eq:matched-alpha-objectives}) with $\alpha_{\mathrm{eff}}=k/m$, scaled by $1/m$. At a fixed corrupted context, this scaling leaves the optimum unchanged. As $k\to0$, the loss reduces to mean token cross-entropy.}

\newrevision{For $k>0$, the gradient is the mean-token CE gradient multiplied by $\exp(k s_\theta)$ (Appendix~\ref{app:population-optima}). Larger $k$ gives less weight to examples whose target tokens receive low probabilities. We therefore initialize training from a pretrained denoiser.}

\newrevision{To relate this training rule to Section~\ref{sec:mode-selection}, substitute $\alpha=k/m$. Equation~\ref{eq:factorized-target-tradeoff} applies when $0<k<m$. Under the assumptions of Proposition~\ref{prop:compatible-mixtures}, $1<k<m$ guarantees that the optimal predictor assigns positive probability to every completion in one group and none outside it. For $k\ge m$, \uline{predicting a single most likely completion is optimal} (Appendix~\ref{app:population-optima}).}

\newrevision{At inference, we use the parallel or blockwise samplers described in Section~\ref{sec:parallel-decoding}.}
}

\begin{table}[!b]
\setlength{\abovecaptionskip}{0pt}
\setlength{\belowcaptionskip}{4pt}
\caption{\textbf{GSM8K accuracy (\%) \reviewrevision{$\uparrow$} at fixed NFE.} Recomputed MDLM, DUO, and $\mathbb{S}$-FLM results and published FMLM+ and DBTM values. \presentationrevision{IDLM values are reevaluated at 4--16 NFE and taken from the publication at 32--64 NFE.} \presentationrevision{MDLM (confidence) and AlphaDLM use the same fixed-NFE confidence sampler. Other baselines use their original samplers.} \presentationrevision{Bold values are the column maxima.} \uline{Protocols} \presentationrevision{are detailed} in Appendix~\ref{app:tinygsm-protocols}.}
\label{tab:tinygsm-published}
\centering
{\small\setlength{\tabcolsep}{3pt}\begin{tabularx}{\linewidth}{@{}l*{5}{>{\raggedleft\arraybackslash}X}@{}}
\toprule
Method / NFE & 4 & 8 & 16 & 32 & 64 \\
\midrule
MDLM (ancestral) \presentationrevision{\citep{sahoo2024mdlm}} & 1.7 & 5.3 & 13.0 & 22.1 & 28.2 \\
DUO ($T=0.1$) \presentationrevision{\citep{sahoo2025duo}} & 5.2 & 14.0 & 23.0 & 30.8 & 31.2 \\
$\mathbb{S}$-FLM (exact, $T=0.1$) \presentationrevision{\citep{deschenaux2026sflm}} & 7.6 & 12.6 & 15.1 & 16.9 & 16.9 \\
$\mathbb{S}$-FLM (top-1) & 7.7 & 12.4 & 14.9 & 16.5 & 17.4 \\
FMLM+ \presentationrevision{\citep{agarwal2026posterior}} & 2.9 & 8.7 & 13.4 & 19.0 & 19.1 \\
FMLM+ (Distill) & 3.9 & 10.3 & 18.7 & 21.6 & 23.4 \\
FMLM+ (Init) & 5.1 & 15.1 & 26.1 & 31.8 & 33.6 \\
DBTM \presentationrevision{\citep{tang2026dbtm}} & 5.7 & 10.1 & 14.3 & 16.8 & 16.2 \\
IDLM (MDLM) \presentationrevision{\citep{li2026idlm}} & \presentationrevision{0.4} & \presentationrevision{1.8} & \presentationrevision{6.5} & 12.8 & 14.9 \\
IDLM (DUO) & \presentationrevision{2.3} & \presentationrevision{7.1} & \presentationrevision{11.7} & 15.4 & 19.0 \\
\midrule
MDLM (confidence) & 4.2 & 20.1 & 37.7 & \newrevision{43.9} & \newrevision{40.6} \\
\rowcolor{statementfill}
\newrevision{\textbf{AlphaDLM (ours, $k=8$)}} & \newrevision{28.9} & \newrevision{\textbf{44.0}} & \newrevision{\textbf{49.2}} & \newrevision{\textbf{49.9}} & \newrevision{\textbf{50.0}} \\
\rowcolor{statementfill}
\textbf{AlphaDLM (ours, $k=16$)} & \textbf{34.6} & 39.0 & 40.6 & 40.9 & 40.9 \\
\bottomrule
\end{tabularx}
}
\end{table}

\section{Experiments}
\label{sec:experiments}
\newrevision{We compare training objectives in the TinyGSM-to-GSM8K setting, a common evaluation setup in recent work on diffusion and flow language models \citep{deschenaux2026sflm,agarwal2026posterior,li2026idlm} (Section~\ref{sec:tinygsm}). We then test transfer to SDAR-1.7B on code and mathematics (Section~\ref{sec:sdar-math}).} TinyGSM results use task accuracy versus the number of model evaluations (NFE). SDAR results use accuracy versus tokens per forward (TPF), defined as returned response tokens divided by denoising forwards plus one prompt-prefill forward per example.

\subsection{TinyGSM}
\label{sec:tinygsm}

\paragraph{Training setup.}
\presentationrevision{We train on TinyGSM~\citep{liu2023tinygsm} and evaluate Python solutions on all 1,319 GSM8K test problems~\citep{cobbe2021verifiers}, following the protocol of $\mathbb{S}$-FLM and FMLM+~\citep{deschenaux2026sflm,agarwal2026posterior}. Objective ablations start from a shared $k=1$ checkpoint at 200k updates and continue for 50k updates with the same architecture and optimizer settings. We test sequence-level $k\in\{1,2,4,8,16\}$ and token-wise $\alpha\in\{0.25,0.375,0.5,0.75\}$, with larger $k$ tested for one-evaluation generation. We use $\alpha_{\mathrm{eff}}=k/m$ (Equation~\ref{eq:effective-alpha}) and average $\ell_k$ over data and corruptions without time weighting. Each recipe uses one training seed and EMA weights. Appendix~\ref{app:tinygsm-protocols} gives \uline{the protocols}.}

\paragraph{Sampling protocol.}
\presentationrevision{We use fixed-NFE confidence sampling (Section~\ref{sec:parallel-decoding}) with the FMLM+ token-reveal schedule~\citep{agarwal2026posterior}. Each round reveals positions above confidence 0.999, adding highest-confidence positions as needed to distribute the remaining masks across the remaining rounds. Tokens are chosen by argmax. Regardless of the threshold, every example executes all prescribed forwards, even if all masks are filled earlier.}

\Needspace{7\baselineskip}
\noindent\presentationrevision{\textbf{Comparison with existing methods.} At four NFE, our AlphaDLM ($k=16$) reaches 34.6\% accuracy, compared with 4.2\% for MDLM using the same sampler and at most 7.7\% for the other baselines shown (Figure~\ref{fig:tinygsm-teaser}, Table~\ref{tab:tinygsm-published}).}

\presentationrevision{To isolate the effect of applying alpha loss jointly rather than token-wise, we compare runs with the same initialization, training budget, and sampler (Figure~\ref{fig:tinygsm-matched-budget}~(a,b)). At four NFE, sequence-level $k=16$ reaches 34.6\%, compared with 8.1\% for the $k=1$ continuation and 19.9\% for the best tested token-wise setting. At sixteen NFE, sequence-level $k=8$ reaches 49.2\%, compared with 44.4\% for the best token-wise setting. The best token-wise parameter is selected separately at each budget.}
\par\nopagebreak[4]
\begin{experimenttakeaway}
\presentationrevision{\textbf{With the same initialization, training budget, and sampler, sequence-level alpha loss gives higher accuracy than every tested token-wise setting at four and sixteen NFE.}}
\end{experimenttakeaway}

\par\medskip\noindent\begin{minipage}{\linewidth}
\setlength{\abovecaptionskip}{4pt}
\centering
\includegraphics[width=0.95\linewidth,trim=0 141.6bp 0 0,clip]{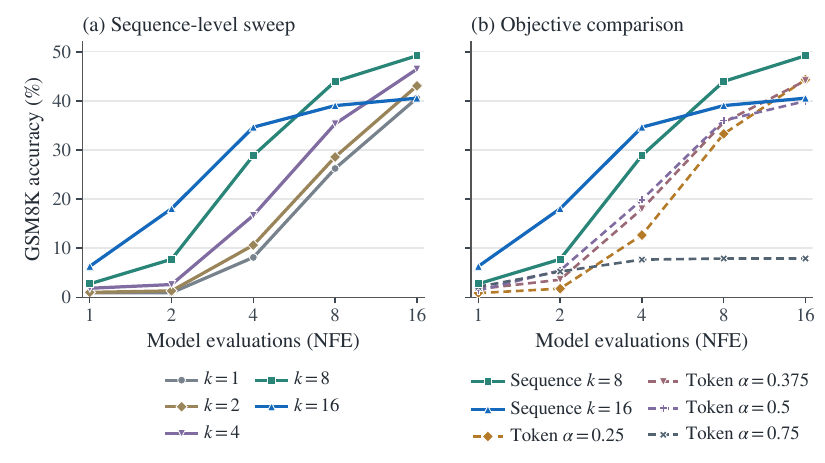}
\captionof{figure}{\presentationrevision{\textbf{How the objective changes parallel prediction.} (a) Sequence-level sweep. (b) Token-wise comparison. All runs continue the same 200k-update $k=1$ checkpoint to 250k using fixed-NFE confidence sampling. Values: Table~\ref{tab:tinygsm-matched-budget}. Panels (c,d) continue on page~\pageref{fig:tinygsm-one}.}}
\label{fig:tinygsm-matched-budget}

\end{minipage}\par\medskip

\Needspace{4\baselineskip}
\presentationrevision{\textbf{Choosing $k$.} The advantage of $k=16$ at four NFE does not persist as we allow more evaluations. At sixteen NFE, it reaches 40.6\%, matching $k=1$ but trailing $k=8$. Each curve uses one checkpoint, so this reversal shows that the best tested $k$ depends on the decoding budget. }

\Needspace{4\baselineskip}
\presentationrevision{To check whether the comparison depends on other experimental choices, we also evaluate \uline{adaptive confidence sampling} (Appendix~\ref{app:tinygsm-adaptive}) and \uline{CE initialization} (Appendix~\ref{app:tinygsm-initialization}). With CE initialization, sequence-level training still outperforms the tested token-wise settings at four and sixteen NFE.}

\Needspace{8\baselineskip}
\presentationrevision{\textbf{One-evaluation generation.} We next test whether this advantage persists when all tokens are predicted in a single model evaluation. Every masked position receives its argmax token, so a correct solution requires the independently predicted tokens to form a jointly correct answer (Figure~\ref{fig:tinygsm-one}~(c,d)). At $k=16$, our method solves 83 of 1,319 problems (6.29\%), compared with 28 (2.12\%) for the best tested token-wise setting. Accuracy increases from 0.91\% at $k=1$ to 101/1,319 (7.66\%) at $k=32$, then falls to 6.75\% at $k=64$. These results remain below multi-step accuracy and do not identify the true conditional mode.}

\par\medskip\noindent\begin{minipage}{\linewidth}
\centering
\phantomsection
\makeatletter\edef\@currentlabel{\getrefnumber{fig:tinygsm-matched-budget}}\makeatother
\label{fig:tinygsm-one}
\includegraphics[width=0.95\linewidth,trim=0 0 0 224bp,clip]{figures/tinygsm_objectives.pdf}
\par\smallskip
\raggedright Figure~\ref{fig:tinygsm-one} (continued): \textbf{One-evaluation generation.} (c,d) Exact-argmax accuracy for sequence-level and token-wise sweeps from the same initialization.
\end{minipage}\par\medskip

\Needspace{10\baselineskip}
\presentationrevision{\textbf{Ancestral sampling.} We also test whether the improvement depends on confidence-based position selection. In this sampler, reveal positions are random and independent of confidence (Figure~\ref{fig:tinygsm-ancestral}, Table~\ref{tab:tinygsm-fixed}). At 32 NFE, increasing $k$ from 1 to 16 raises accuracy from 8.2\% to 32.5\% with sampled tokens ($T=1$), and from 26.0\% to 37.1\% with argmax tokens ($T=0$). The gain at $T=0$ also rules out scalar logit sharpening as the sole explanation, since sharpening changes neither argmax tokens nor reveal probabilities. Appendix~\ref{app:passk} reports \uline{ancestral pass@$K$}.}
\par\nopagebreak[4]
\begin{experimenttakeaway}
\presentationrevision{\textbf{The gains persist when all tokens are predicted at once and when reveal positions are sampled independently of confidence.}}
\end{experimenttakeaway}

\par\medskip\noindent\begin{minipage}{\linewidth}
\centering
\includegraphics[width=\linewidth,trim=0 12bp 0 0,clip]{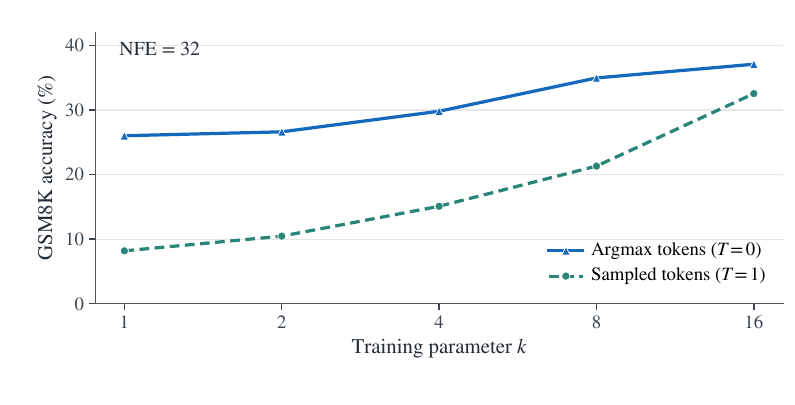}
\captionof{figure}{\presentationrevision{\textbf{Ancestral sampling at 32 NFE.} Reveal positions are random and independent of confidence. Tokens are selected by argmax ($T=0$) or sampled ($T=1$).} Values: Table~\ref{tab:tinygsm-fixed}.}
\label{fig:tinygsm-ancestral}
\end{minipage}\par\medskip

\Needspace{12\baselineskip}
\subsection{Transfer to SDAR on mathematics and code}
\label{sec:sdar-math}

\presentationrevision{We next test whether sequence-level alpha loss also improves blockwise generation with SDAR-1.7B-Chat~\citep{cheng2025sdar}. We fine-tune on Nemotron-SFT-Math-v4 for mathematics and OpenCodeInstruct for code, and compare our AlphaDLM with the original model, CE fine-tuning, and entropy-regularized fine-tuning (CE+CAP)~\citep{bie2025llada2}. We apply Equation~\ref{eq:effective-alpha} separately to each four-token response block. All methods use adaptive confidence sampling following Fast-dLLM~\citep{wu2025fastdllm} within four-token blocks (Section~\ref{sec:parallel-decoding}), with argmax tokens. We vary the threshold to compare accuracy and TPF. Figure~\ref{fig:sdar-math}~(a--d) uses $k=0.8$ for mathematics and $k=0.4$ for code. Panels (e,f) show the mathematics parameter sweep. Appendix~\ref{app:sdar-protocols} gives \uline{training and evaluation details}, including block-loss aggregation.}

\presentationrevision{\textbf{Mathematics.} At threshold 0.99, our method improves both accuracy and TPF over CE (Figure~\ref{fig:sdar-math}~(a,b)). Accuracy reaches 51.5\% on MATH500 and 77.2\% on GSM8K, compared with 50.6\% and 76.3\% for CE. TPF increases by 31.6\% and 27.1\%, respectively.}

\par\medskip\noindent\begin{minipage}{\linewidth}
\centering
\includegraphics[width=\linewidth]{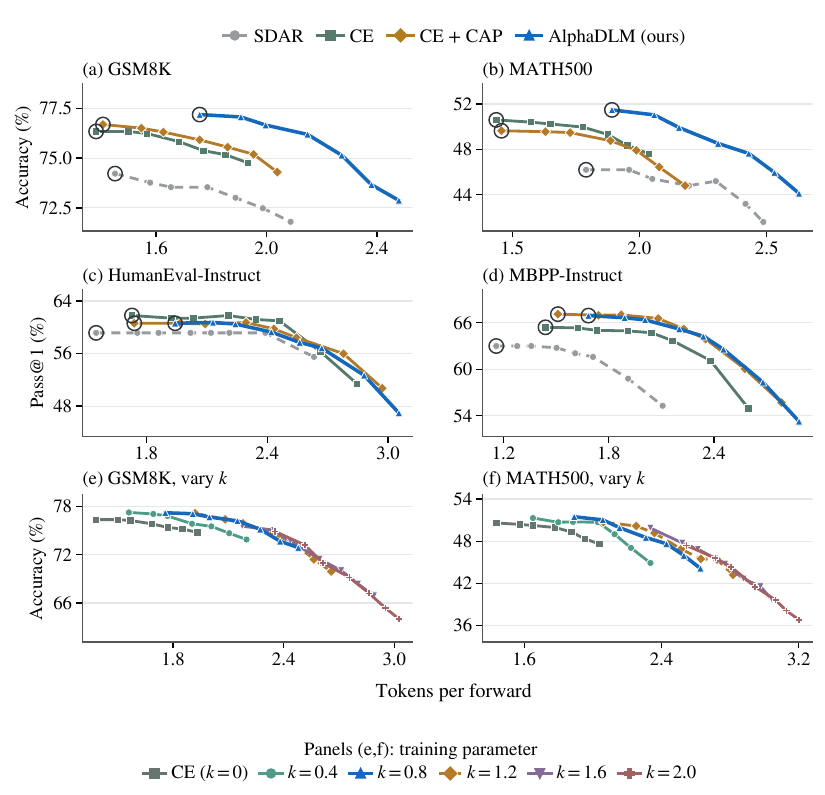}
\captionof{figure}{\textbf{Accuracy--TPF trade-offs with SDAR-1.7B.} (a--d) Objective comparisons on mathematics and code. AlphaDLM (ours) uses $k=0.8$ for mathematics and $k=0.4$ for code. (e,f) Mathematics parameter sweep. \presentationrevision{All panels use adaptive confidence sampling with argmax tokens.} Points average accuracy and TPF across runs at each confidence threshold. \presentationrevision{Rings in (a--d) mark threshold 0.99.} Appendix~\ref{app:sdar-protocols} provides \presentationrevision{\uline{standard-deviation bands} (Figures~\ref{fig:sdar-uncertainty}--\ref{fig:sdar-ablation-uncertainty})}.}
\label{fig:sdar-math}
\label{fig:sdar-code}
\label{fig:sdar-ablation}
\end{minipage}\par\medskip

\Needspace{8\baselineskip}
\presentationrevision{\textbf{Code.} At the same threshold, our method increases TPF over CE+CAP with similar pass@1 (Figure~\ref{fig:sdar-code}~(c,d)). On HumanEval-Instruct, both achieve 60.6\% pass@1, while TPF increases from 1.74 to 1.94. On MBPP-Instruct, TPF increases from 1.51 to 1.68 with a 0.19-percentage-point drop in pass@1. This comparison concerns threshold 0.99. Across the full sweep, peak accuracy is highest for CE on HumanEval-Instruct and CE+CAP on MBPP-Instruct.}

\presentationrevision{\textbf{Choosing $k$.} The mathematics sweep shows that larger $k$ can increase TPF at the cost of peak accuracy (Figure~\ref{fig:sdar-ablation}~(e,f)). Among the tested values, $k=0.4$ and $k=0.8$ give the highest peak accuracies. Increasing $k$ to 1.6 or 2.0 increases TPF but reduces peak accuracy on both benchmarks.}

\FloatBarrier

\Needspace{4\baselineskip}
\section{Discussion and Conclusion}
\reviewrevision{AlphaDLM applies alpha loss to the joint probability of masked tokens, interpolating between cross-entropy and joint-mode fitting. Our analysis explains why this sequence-level objective behaves differently from applying the same loss to individual tokens. For $0<\alpha<1$, alpha loss lowers the target temperature without factorization. With factorization, the optimal denoiser fits marginals of a joint distribution that balances closeness to this target against dependence between tokens. Under the conditions of Proposition~\ref{prop:compatible-mixtures}, it can exclude invalid combinations while preserving several valid completions.}

\reviewrevision{The experiments show that sequence-level alpha loss improves low-NFE accuracy while keeping the factorized architecture. The gains persist in fully parallel prediction and ancestral sampling, extending beyond confidence-based position selection. Together, these results show that the limitations of parallel prediction depend on the combination of factorization and the training objective. Alpha loss provides a way to improve this combination without explicitly modeling token dependencies within each denoising evaluation.}

\setlength{\bibsep}{3pt plus 1pt}
\bibliographystyle{iclr2027_conference}
\bibliography{alphadlm_references}

\clearpage
\appendix
\Needspace{12\baselineskip}
\section{Related Work}
\label{app:related}
\begingroup
\setlength{\parskip}{2pt}
\paragraph{Diffusion language models and decoding.}
Masked diffusion models learn token predictions under a corruption process \citep{sahoo2024mdlm,shi2024md4,nie2025llada}. Confidence-based decoding controls how many tokens are revealed at each evaluation \citep{wu2025fastdllm}, while SDAR combines parallel prediction within blocks with autoregression across blocks \citep{arriola2025block,cheng2025sdar}. \newrevision{Our method, AlphaDLM,} changes the training objective and can be used with either decoding structure.

\paragraph{Training for few-step generation.}
Several methods adapt diffusion language models for generation with fewer model evaluations. T3D \citep{zhang2026t3d} combines trajectory self-distillation with a reverse-KL-inspired discriminative objective and analyzes how trajectory supervision reduces conditional token dependence. dParallel \citep{chen2026dparallel} combines trajectory supervision with entropy minimization on correctly predicted tokens to encourage earlier parallel commitments. CDLM \citep{kim2026cdlm} distills teacher trajectories into a block-causal student using teacher-distribution matching and consistency objectives. \newrevision{Our method} shares the goal of improving generation under limited decoding budgets, but directly applies sequence-level alpha loss to training examples and their corruptions, without requiring teacher trajectories or distribution matching. Our analysis characterizes how this objective changes the joint target fitted by a factorized denoiser and establishes conditions under which its optimum preserves multiple valid completions while excluding invalid combinations.

\paragraph{Modeling token dependencies.}
Parallel prediction and its multimodality challenge predate diffusion language models, including non-autoregressive translation \citep{gu2018nat,huang2022nat}. For discrete diffusion, Discrete Copula Diffusion \citep{liu2025copula} supplements the denoiser with a generative model that supplies dependency information. EDLM \citep{xu2025energy} introduces a sequence-level energy correction. \citet{kim2026tensor} represent the conditional clean distribution through low-rank tensor decompositions, while \citet{bansal2025joint} train a lightweight sampler on top of a frozen diffusion model to approximate joint sampling. These approaches enrich the modeled dependencies or the sampling procedure. \newrevision{Our method uses} the factorized denoiser and changes its training objective. Our analysis asks which joint target this restricted predictor fits and when independent predictions can exclude invalid combinations without \newrevision{concentrating on a single completion}.

\paragraph{Evaluation and baselines.}
TinyGSM \citep{liu2023tinygsm} is a shared setting for generating mathematical solutions as programs. $\mathbb{S}$-FLM \citep{deschenaux2026sflm} compares hyperspherical flows with MDLM and DUO \citep{sahoo2025duo}. FMLM+ \citep{agarwal2026posterior} uses posterior refinement, DBTM \citep{tang2026dbtm} learns a discrete transport map, and IDLM \citep{li2026idlm} distills a pretrained diffusion teacher. Our comparison \newrevision{uses} their respective generation procedures. \reviewrevision{\newrevision{Our AlphaDLM} uses \presentationrevision{fixed-NFE confidence sampling} with the token-reveal schedule of FMLM+, while its objective ablations hold the sampler fixed.} Appendix~\ref{app:tinygsm-protocols} \newrevision{documents the \uline{baseline sources and sampling protocols}}.

\paragraph{Objectives and dependence.}
\looseness=-1 Generalized cross-entropy interpolates between cross-entropy and mean absolute error in classification \citep{zhang2018gce}. \reviewrevision{\citet{li2026beyond} study probability-based objectives for autoregressive supervised fine-tuning, including the same alpha-loss family applied \newrevision{\underline{token-wise}} to individual token probabilities. They show that downweighting low-probability tokens can help when the pretrained model already has strong task-relevant capabilities, while cross-entropy performs better when these capabilities are weak. We apply alpha loss to the joint probability of masked tokens and analyze how factorization changes the distribution that minimizes this sequence-level objective.} \citet{huang2022nat} analyze dependence lost by non-autoregressive marginal fitting and interpret alternative objectives through proxy targets. \citet{minka2005divergence} shows how the divergence shapes a factorized approximation, and \citet{lapidoth2019dependence} establish variational identities for product-distribution optimization. Our analysis connects this viewpoint to denoising objectives, characterizing the selected joint target, distinguishing sequence-level and token-wise fitting, and \newrevision{establishing conditions under which a factorized predictor assigns probability only to valid completions}. Recent analyses relate unmasking schedules and parallel-sampling error to total correlation \citep{zhao2026intrinsic,wen2026serial}. Our characterization concerns the joint target selected by training. Token-wise alpha loss and entropy-regularized SFT provide relevant empirical controls because both also change prediction confidence.

\par\endgroup

\clearpage
\section{Population optimum of sequence-level alpha loss}
\label{app:population-optima}
{
\begingroup\color{revisiontext}
We explain why alpha loss fits a lower-temperature target when the predictor is unrestricted, and what changes when it predicts tokens independently. The key step is to separate two questions: which joint distribution to fit, and how to predict its individual tokens. Appendix~\ref{app:compatible-support} proves Proposition~\ref{prop:compatible-mixtures} directly from the expected alpha loss, using an elementary argument illustrated on the two-token example.

Throughout, we fix the visible context. A sequence $\cleanvar{y}=(\cleanvar{y}_1,\ldots,\cleanvar{y}_m)$ fills the $m$ masked positions, $q$ is its true conditional distribution, and $p$ is the predictor. The space of sequences is finite. We optimize over distributions directly, without neural-network constraints, and initially assume $0<\alpha<1$. All logarithms are natural.

The expected loss depends on $p$ only through the following weighted sum:
\begin{equation}
 \mathbb E_{\cleanvar{y}\sim q}[\ell^{\mathrm{seq}}_\alpha(p;\cleanvar{y})]
 =\frac{1-F_\alpha(p)}{\alpha},
 \qquad F_\alpha(p)=\sum_{\cleanvar{y}}q(\cleanvar{y})p(\cleanvar{y})^\alpha.
 \label{eq:alpha-fitting-objective}
\end{equation}
Thus minimizing the loss means maximizing $F_\alpha$. Taking its logarithm also leaves the maximizing predictor unchanged.

\paragraph{Without factorization: the optimum is the lower-temperature target.}
\begin{proof}[Derivation of Equation~\eqref{eq:joint-temperature-optimum}]
We first allow $p$ to be any joint distribution. At an optimum, moving a small amount of probability from one sequence to another cannot improve $F_\alpha$. The gain from adding probability must therefore be the same for every sequence with $q(\cleanvar{y})>0$:
\begin{equation}
 \frac{\partial F_\alpha}{\partial p(\cleanvar{y})}
 =\alpha q(\cleanvar{y})p(\cleanvar{y})^{\alpha-1}
 =\lambda
 \quad\Longrightarrow\quad
 p(\cleanvar{y})\propto q(\cleanvar{y})^{1/(1-\alpha)}.
 \label{eq:positive-power-optimum}
\end{equation}
Here $\lambda$ is the common gain. Normalizing the probabilities gives exactly $p=q_T$, with $T=1-\alpha$, as defined in Equation~\ref{eq:tempered-target}.

Two facts ensure that this calculation finds the unique global optimum. First, an optimum puts no probability where $q=0$, since moving that probability to a sequence with $q>0$ increases $F_\alpha$. It puts positive probability on every sequence with $q>0$, since the gain in Equation~\ref{eq:positive-power-optimum} grows without bound as its predicted probability approaches zero. Second, $x^\alpha$ is strictly concave for $0<\alpha<1$. The weighted sum $F_\alpha$ is therefore strictly concave on this support, so the stationary solution is the unique maximum.
\end{proof}

This is the optimum in \citet[Proposition~1]{sypherd2022tunable}. Their parameter $\beta=1/(1-\alpha)$ gives $1-1/\beta=\alpha$ and $\beta/(\beta-1)=1/\alpha$, recovering both our loss and this optimum.

\paragraph{With factorization: the optimum fits the marginals of a selected target.}
\begin{proof}[Derivation of Equation~\eqref{eq:factorized-target-tradeoff}]
Now $p(\cleanvar{y})=\prod_i p_i(\cleanvar{y}_i)$. To fit the positions separately, we rewrite the objective using log-probabilities, which turn a product into a sum. An auxiliary joint distribution $\pi$ will supply the weights for this fit. It is used only in the proof. The following three steps specialize the product-distribution identity of \citet{lapidoth2019dependence} to $m$ positions.

\textbf{1. The auxiliary distribution gives an equivalent objective.}
For a fixed predictor with $F_\alpha(p)>0$, normalize the terms in its objective to obtain a distribution:
\begin{equation*}
 \pi_p(\cleanvar{y})=\frac{q(\cleanvar{y})p(\cleanvar{y})^\alpha}{F_\alpha(p)}.
\end{equation*}
For any joint distribution $\pi$ supported where $q>0$, expanding the definition of KL divergence gives
\begin{equation}
 \alpha\,\mathbb E_{\cleanvar{y}\sim\pi}\log p(\cleanvar{y})-D_{\mathrm{KL}}(\pi\Vert q)
 =\log F_\alpha(p)-D_{\mathrm{KL}}(\pi\Vert\pi_p).
 \label{eq:alpha-kl-rewrite}
\end{equation}
KL divergence is nonnegative and is zero exactly when its two distributions agree. Thus the left-hand side is at most $\log F_\alpha(p)$, with equality at $\pi=\pi_p$. Maximizing it over $\pi$ recovers the original objective exactly. This is the finite-distribution form of the Gibbs variational identity~\citep[Theorem~3.4]{wainwright2008graphical}.

\textbf{2. The best token predictions are the target marginals.}
We can now maximize over $p$ and $\pi$ together. Fix $\pi$ first. Factorization makes the log-probability a sum over positions:
\begin{equation}
 \mathbb E_{\cleanvar{y}\sim\pi}\log p(\cleanvar{y})
 =\sum_i\sum_{\cleanvar{y}_i}\pi_i(\cleanvar{y}_i)\log p_i(\cleanvar{y}_i)
 =-\sum_i H(\pi_i)-\sum_i D_{\mathrm{KL}}(\pi_i\Vert p_i).
\end{equation}
Here $\pi_i$ is the token marginal and $H(r)=-\sum_x r(x)\log r(x)$ measures uncertainty in a distribution $r$. Each position is an ordinary cross-entropy fitting problem, solved by $p_i=\pi_i$. Substituting these optimal token distributions leaves only the choice of $\pi$:
\begin{equation}
 \max_{p=\prod_i p_i}\log F_\alpha(p)
 =-\min_\pi J(\pi),
 \qquad J(\pi):=D_{\mathrm{KL}}(\pi\Vert q)+\alpha\sum_i H(\pi_i).
 \label{eq:variational-alpha-identity}
\end{equation}
Thus the optimal denoiser fits the marginals of a distribution minimizing $J$.

\textbf{3. The selected target balances fit and token dependence.}
The sum of token entropies equals the joint entropy plus total correlation:
\begin{equation*}
 \sum_i H(\pi_i)=H(\pi)+\operatorname{TC}(\pi).
\end{equation*}
Total correlation is the dependence measure defined in Section~\ref{sec:mode-selection}. To express $J$ using the lower-temperature target, write $q_T(\cleanvar{y})=q(\cleanvar{y})^{1/(1-\alpha)}/Z_T$, where $Z_T=\sum_{\cleanvar{y}}q(\cleanvar{y})^{1/(1-\alpha)}$. Using $\log q=(1-\alpha)(\log q_T+\log Z_T)$ and $D_{\mathrm{KL}}(\pi\Vert q)=-H(\pi)-\mathbb E_\pi\log q$ gives
\begin{equation}
 J(\pi)=(1-\alpha)D_{\mathrm{KL}}(\pi\Vert q_T)
       +\alpha\operatorname{TC}(\pi)-(1-\alpha)\log Z_T.
\end{equation}
The last term is constant in $\pi$. Removing it and dividing by $1-\alpha>0$ gives the claimed characterization:
\begin{equation*}
 \pi_\alpha^*\in\operatorname*{arg\,min}_{\pi}
 \left\{D_{\mathrm{KL}}(\pi\Vert q_T)
       +\frac{\alpha}{1-\alpha}\operatorname{TC}(\pi)\right\},
 \qquad p_\alpha^{*,\mathrm{fact}}=\prod_i\pi_{\alpha,i}^*.
\end{equation*}
Both optimization steps are exact. Every minimizing $\pi_\alpha^*$ gives an optimal predictor through its marginals. Conversely, any optimal predictor $p^*$ paired with $\pi_{p^*}$ attains the joint maximum in Steps 1--2. It must therefore satisfy $p_i^*=\pi_{p^*,i}$, and $\pi_{p^*}$ must minimize $J$. This proves the statement for all global optima.
\end{proof}

\textbf{Zero probabilities.} We use $0\log0=0$ and set KL divergence to infinity when its first argument assigns probability outside the support of its second. In Equation~\ref{eq:alpha-kl-rewrite}, a distribution $\pi$ assigning mass where $p=0$ gives value $-\infty$ on the left and cannot maximize it. A predictor with $F_\alpha(p)=0$ cannot be optimal, since predicting any sequence with $q>0$ gives a positive value. Thus zeros require no positivity assumption on the target or the optimal predictor.

\paragraph{Token-wise loss and the endpoints.}
For token-wise loss, the expected objective is a sum of separate objectives for individual positions. Applying the unrestricted calculation to each marginal $q_i$ gives Equation~\ref{eq:tokenwise-temperature-optimum}. If $q$ already factorizes, so does $q_T$. Taking $\pi=q_T$ then makes both terms in Equation~\ref{eq:factorized-target-tradeoff} zero, so sequence-level and token-wise loss have the same optimum.

At $\alpha=0$, both losses for a factorized predictor become the sum of token cross-entropies, whose optimal factors are the marginals of $q$. For $\alpha\ge1$, use $p(\cleanvar{y})^\alpha\le p(\cleanvar{y})$ to obtain
\begin{equation*}
 F_\alpha(p)\le\sum_{\cleanvar{y}}q(\cleanvar{y})p(\cleanvar{y})
 \le\max_{\cleanvar{y}}q(\cleanvar{y}).
\end{equation*}
Predicting a most likely sequence with probability one attains this bound and is possible with a factorized predictor. It is the unique optimum when the mode is unique. At $\alpha=1$, any joint distribution supported on tied modes is also optimal if it belongs to the allowed predictor class. For $\alpha>1$, equality requires a point mass on one mode.

These statements optimize distributions at one fixed context. Taking a logarithm preserves this optimum, but does not generally preserve an objective averaged over contexts with shared model parameters. Appendix~\ref{app:compatible-support} gives the additional support condition needed to exclude invalid combinations.
\endgroup

\paragraph{Numerical implementation.}
\reviewrevision{The mask-normalized sequence-level loss has a shared gradient weight determined by the joint probability of the original tokens at all masked positions:
\begin{equation}
 \nabla_\theta\left(\frac{\ell^{\mathrm{seq}}_\alpha}{m}\right)
 =p_\theta(\cleanvar{x}_M\mid \statevar{z},c)^\alpha\,\nabla_\theta\ell_{\mathrm{CE}},
 \qquad
 \ell_{\mathrm{CE}}=-\frac1m\log p_\theta(\cleanvar{x}_M\mid \statevar{z},c).
 \label{eq:alpha-loss-gradient}
\end{equation}
All missing-token predictions share this weight, whereas token-wise alpha loss assigns a separate weight to each position. Under $\alpha_{\mathrm{eff}}=k/m$, the shared weight is $\exp(k s_\theta)$. Small target probabilities suppress this weight more strongly at large $k$, motivating a warm start. Averaging log-probabilities makes $k$ comparable across mask counts, but does not remove this suppression.}

The implemented loss evaluates $1-\exp(k s_\theta)$ as $-\operatorname{expm1}(k s_\theta)$ to avoid cancellation near zero. Examples without active positions contribute zero.
}
\raggedbottom
\subsection{Groups of valid completions at intermediate alpha}
\label{app:compatible-support}
\begingroup\color{revisiontext}
The example in Figure~\ref{fig:alpha-interpolation} has two groups of valid completions: $\{\texttt{I}\}\times\{\texttt{am}\}$ and $\{\texttt{he},\texttt{she}\}\times\{\texttt{is}\}$. Combining tokens within either group is safe, while mixing the groups can produce invalid answers. We prove directly that alpha loss can select one group while giving positive probability to every completion within it. The proof uses only probabilities, powers, and the arithmetic--geometric mean inequality.

Selection of a single group was established for two tokens with one completion per group by \citet[Section~3.4]{minka2005divergence} and \citet[Lemma~32]{lapidoth2019dependence}. Here each group may contain multiple completions, with arbitrary dependencies in their target probabilities.

\paragraph{The support condition.}
For position $i$ and group $g$, let $A_{ig}$ be the allowed token set. The valid sequences form $G$ nonempty groups
\begin{equation}
 \{\cleanvar{y}:q(\cleanvar{y})>0\}=\bigcup_{g=1}^G S_g,
 \qquad S_g=A_{1g}\times\cdots\times A_{mg},
 \qquad A_{ig}\cap A_{ig'}=\varnothing\quad(g\ne g').
 \label{eq:compatible-support}
\end{equation}
Thus every combination within a group is valid, and a token at any one position identifies its group. The index $g$ refers to a group of completions, not a decoding block.
\endgroup

\begin{mdframed}[style=propositionstyle]
\textbf{Proposition~\ref{prop:compatible-mixtures} (Multiple valid completions, restated).} Consider $m\ge2$ masked tokens, and suppose the valid completions split into groups such that any within-group combination of tokens is valid, and different groups use disjoint token sets at every position. For $1/m<\alpha<1$, every globally optimal factorized predictor selects one group and assigns positive probability to all completions in that group.
\end{mdframed}

\begingroup\color{revisiontext}
\paragraph{Setting and running example.}
We fix the visible context and write $p(\cleanvar{y})=\prod_i p_i(\cleanvar{y}_i)$. By Equation~\ref{eq:alpha-fitting-objective}, minimizing expected alpha loss is equivalent to maximizing
\begin{equation*}
 F_\alpha(p)=\sum_{\cleanvar{y}}q(\cleanvar{y})p(\cleanvar{y})^\alpha.
\end{equation*}
The five steps below first show that an optimum selects one group, then that it covers the whole group. The figures use the same target as Figures~\ref{fig:marginals-vs-modes} and~\ref{fig:alpha-interpolation}, with $\alpha=0.55$ strictly inside the proposition's range for $m=2$. The example at the boundary $\alpha=1/2$ is treated separately below.

Figure~\ref{fig:prop1-setting} starts from the cross-entropy optimum, whose factors are the marginals of $q$. It assigns $0.48$ probability to invalid completions. We will compare this predictor with the predictors obtained by restricting it to either valid group.
\proofGroupSetupFigure

\paragraph{Conditioning on a group.}
Write $p_i(A_{ig})$ for the probability that position $i$ chooses a token from group $g$. Independence gives
\begin{equation*}
 p(S_g)=\prod_{i=1}^m p_i(A_{ig}),
 \qquad \sum_g p_i(A_{ig})\le1.
\end{equation*}
The inequality allows tokens outside all valid groups. If $p(S_g)>0$, conditioning on $S_g$ gives another factorized predictor, denoted $p^{(g)}$, with factors $p_i^{(g)}=p_i(\,\cdot\mid A_{ig})$. It gives probability only to sequences in $S_g$.

\begin{proof}[Proof of Proposition~\ref{prop:compatible-mixtures}]
\textbf{Step 1. Split the objective over groups.}
The target is zero outside the valid groups. Within a group of positive probability, $p(\cleanvar{y})=p(S_g)p^{(g)}(\cleanvar{y})$. Substitution into the objective gives
\begin{equation}
 F_\alpha(p)=\sum_{g:\,p(S_g)>0}p(S_g)^\alpha F_\alpha(p^{(g)}).
 \label{eq:prop1-group-split}
\end{equation}
Groups with $p(S_g)=0$ contribute zero. Every displayed $F_\alpha(p^{(g)})$ is positive, since $p^{(g)}$ gives positive probability to at least one valid completion. These values use the original $q$, without renormalizing it within a group.

Thus $F_\alpha(p)$ is a weighted sum of the values obtained by conditioning on individual groups. We next show that these weights sum to less than one unless $p$ already selects one group.

\Needspace{6\baselineskip}
\textbf{Step 2. Splitting probability between groups reduces the total weight.}
For $\alpha>1/m$ and any group, raising a number in $[0,1]$ to a larger power makes it no larger. The arithmetic--geometric mean inequality then gives
\begin{equation*}
 p(S_g)^\alpha\le p(S_g)^{1/m}
 =\left(\prod_{i=1}^m p_i(A_{ig})\right)^{1/m}
 \le\frac{1}{m}\sum_{i=1}^m p_i(A_{ig}).
\end{equation*}
Summing over groups and using their disjoint token sets yields
\begin{equation}
 \sum_g p(S_g)^\alpha
 \le\sum_g p(S_g)^{1/m}
 \le\frac{1}{m}\sum_{i=1}^m\sum_g p_i(A_{ig})
 \le1.
 \label{eq:prop1-group-weight-bound}
\end{equation}
The first inequality is strict whenever any $p(S_g)$ lies strictly between zero and one. Consequently, the sum can equal one only when exactly one group has $p(S_g)=1$. Figure~\ref{fig:prop1-weights} illustrates both the bound and the threshold $1/m$.
\proofGroupWeightsFigure
\FloatBarrier

\Needspace{6\baselineskip}
\textbf{Step 3. Every optimum selects one group.}
Let $p$ be an optimal factorized predictor. Its objective is positive, since predicting any valid completion with probability one gives $F_\alpha=q(\cleanvar{y})>0$. Therefore at least one group has positive probability. Let $F_{\max}$ be the largest value of $F_\alpha(p^{(g)})$ among these groups. Steps 1--2 imply
\begin{equation}
 F_\alpha(p)\le F_{\max}\sum_g p(S_g)^\alpha\le F_{\max}.
 \label{eq:prop1-conditioning-bound}
\end{equation}
But $F_{\max}$ is attained by one of the factorized predictors $p^{(g)}$. Optimality of $p$ requires $F_\alpha(p)\ge F_{\max}$, so equality must hold throughout. Since $F_{\max}>0$, Step~2 forces $p(S_g)=1$ for one group $g$. Every factor then satisfies $p_i(A_{ig})=1$, because their product is one and none can exceed one.

This proves the first half of the proposition: the predictor assigns no probability outside one valid group. Figure~\ref{fig:prop1-conditioning} shows the improvement from conditioning a predictor that has not yet selected a group.
\proofGroupConditioningFigure
\FloatBarrier

\Needspace{6\baselineskip}
\textbf{Step 4. A token that can contribute cannot have probability zero.}
This step uses $0<\alpha<1$ and applies to any target $q$. Hold all positions except $i$ fixed and group the objective by the token value $\cleanvar{v}\in\mathcal V$ at that position:
\begin{equation}
 F_\alpha(p)=\sum_{\cleanvar{v}\in\mathcal V}a_i(\cleanvar{v})p_i(\cleanvar{v})^\alpha,
 \qquad
 a_i(\cleanvar{v})=\sum_{\cleanvar{y}:\,\cleanvar{y}_i=\cleanvar{v}}
 q(\cleanvar{y})\prod_{j\ne i}p_j(\cleanvar{y}_j)^\alpha.
 \label{eq:within-group-token-fit}
\end{equation}
The coefficient $a_i(\cleanvar{v})$ is positive exactly when this token forms a valid completion with some tokens that the other positions predict with positive probability.

Suppose $a_i(\cleanvar{v})>0$ but $p_i(\cleanvar{v})=0$. Give this missing token probability $\varepsilon$, and multiply all existing probabilities at position $i$ by $1-\varepsilon$, where $0<\varepsilon<1$. The resulting predictor $\widetilde p$ is still factorized. Its objective changes by
\begin{equation*}
 F_\alpha(\widetilde p)-F_\alpha(p)
 =a_i(\cleanvar{v})\varepsilon^\alpha
  -\bigl[1-(1-\varepsilon)^\alpha\bigr]F_\alpha(p).
\end{equation*}
For $\alpha<1$, $(1-\varepsilon)^\alpha\ge1-\varepsilon$, so the loss is at most $\varepsilon F_\alpha(p)$. Hence
\begin{equation}
 F_\alpha(\widetilde p)-F_\alpha(p)
 \ge\varepsilon^\alpha\bigl[a_i(\cleanvar{v})-\varepsilon^{1-\alpha}F_\alpha(p)\bigr]>0
 \label{eq:prop1-missing-token-gain}
\end{equation}
for sufficiently small $\varepsilon$, because $\varepsilon^{1-\alpha}$ tends to zero. This contradicts optimality. Thus every token with $a_i(\cleanvar{v})>0$ has $p_i(\cleanvar{v})>0$: adding a missing valid possibility brings a gain of order $\varepsilon^\alpha$, which outweighs a loss of order $\varepsilon$.

\Needspace{6\baselineskip}
\textbf{Step 5. Every completion in the selected group has positive probability.}
Take a token $\cleanvar{v}\in A_{ig}$ in the group selected in Step~3. At each other position $j$, choose a token $\cleanvar{y}_j\in A_{jg}$ with $p_j(\cleanvar{y}_j)>0$. Such a token exists because $p_j(A_{jg})=1$. With $\cleanvar{y}_i=\cleanvar{v}$, these tokens form a valid completion by the support condition. Its contribution to $a_i(\cleanvar{v})$ is positive, so Step~4 gives $p_i(\cleanvar{v})>0$.

This holds for every token at every position within the selected group. Their products are positive, so every completion in that group has positive probability. Together with Step~3, this proves the proposition. If the selected group contains several completions, all of them remain possible.
\end{proof}

Figure~\ref{fig:prop1-coverage} illustrates Steps~4--5. At $\alpha=0.55$, the best predictor within the second group gives probabilities approximately $0.679$ and $0.321$ to \texttt{he is} and \texttt{she is}, respectively. Its objective is approximately $0.417$, above the value $0.400$ for the only completion in the first group. By Step~3, it is therefore a global optimum.
\proofGroupCoverageFigure
\FloatBarrier

\paragraph{Where the assumptions enter.}
The Cartesian group structure makes conditioning preserve factorization. Disjoint token sets give the probability bound in Step~2. The condition $\alpha>1/m$ makes that bound strict whenever the predictor splits probability between groups. The condition $\alpha<1$ makes a small addition to a missing token beneficial in Step~4. Finally, validity of every within-group combination makes every token eligible for that argument in Step~5. The requirement $m\ge2$ makes the interval $1/m<\alpha<1$ nonempty.

\paragraph{What happens at the threshold?}
The strict lower bound on $\alpha$ is necessary for a guarantee covering all targets with this support. To compare groups, define their best achievable objective values
\begin{equation}
 C_g(\alpha)=\max_{p=\prod_i p_i:\,p(S_g)=1}F_\alpha(p),
 \qquad C^\star=\max_g C_g(\alpha).
 \label{eq:compatible-group-score}
\end{equation}
These maxima exist because the spaces of token probabilities are compact and $F_\alpha$ is continuous.

\textbf{At $\alpha=1/m$.}
The arithmetic--geometric mean bound still gives $\sum_g p(S_g)^{1/m}\le1$. Equation~\ref{eq:prop1-group-split} therefore implies
\begin{equation*}
 F_\alpha(p)\le\sum_g p(S_g)^{1/m}C_g(\alpha)\le C^\star.
\end{equation*}
A best single-group predictor attains $C^\star$. If one group is uniquely best, positive probability on any other group makes the second inequality strict. Equality then forces $p(S_g)=1$ for the best group, and Steps~4--5 ensure that all its completions have positive probability.

If two groups tie, take optimal predictors $p'$ and $p''$ within them and mix their factors at every position: $p_i=(1-s)p'_i+sp''_i$ for $0<s<1$. This is a product of mixed token distributions. Since $m\alpha=1$, its objective is
\begin{equation*}
 F_\alpha(p)=(1-s)^{m\alpha}C^\star+s^{m\alpha}C^\star=C^\star.
\end{equation*}
It is optimal but also generates invalid cross-group combinations. For example, with $q(00)=q(11)=1/2$ and $\alpha=1/2$, independent uniform factors are optimal while assigning half their probability to invalid pairs.

\textbf{Below $1/m$.}
The group decomposition also gives the exact optimal allocation across groups. Let $s_g=\frac1m\sum_i p_i(A_{ig})$ be the mean probability assigned to group $g$ across positions. By arithmetic--geometric mean, $p(S_g)\le s_g^m$, so
\begin{equation*}
 F_\alpha(p)\le\sum_g C_g(\alpha)s_g^{m\alpha},
 \qquad \sum_g s_g\le1.
\end{equation*}
Every allocation with $\sum_g s_g=1$ attains this bound if each position chooses group $g$ with probability $s_g$ and uses the corresponding factor of a best within-group predictor. All $C_g$ are positive, so its maximum uses the full allocation $\sum_g s_g=1$. It remains to maximize this weighted sum of powers. For $0<m\alpha<1$, Equation~\ref{eq:positive-power-optimum} gives the unique optimal allocation
\begin{equation}
 p_i(A_{ig})=s_g^*
 =\frac{C_g(\alpha)^{1/(1-m\alpha)}}{\sum_{g'}C_{g'}(\alpha)^{1/(1-m\alpha)}}>0
 \quad\text{for every position }i.
 \label{eq:prop1-subthreshold-weights}
\end{equation}
To attain the bound, equality in arithmetic--geometric mean requires all positions to use these same group probabilities. With at least two groups, every optimum therefore generates invalid cross-group combinations. At the threshold, ties can allow such combinations. Above it, Proposition~\ref{prop:compatible-mixtures} excludes them under its support condition.

\paragraph{The two-token example at $\alpha=1/2$.}
In Figure~\ref{fig:alpha-interpolation}, the target probabilities of \texttt{I am}, \texttt{he is}, and \texttt{she is} are $0.4$, $0.35$, and $0.25$. The first group has only one completion, so $C_1=0.4$. In the second group, the second token is always \texttt{is}, leaving only the first token to fit. Equation~\ref{eq:positive-power-optimum} gives
\begin{equation*}
 C_2(\alpha)=\left(0.35^{1/(1-\alpha)}+0.25^{1/(1-\alpha)}\right)^{1-\alpha}.
\end{equation*}
At $\alpha=1/2$, $C_2=\sqrt{0.185}>0.4$. The unique optimum therefore chooses the second group, with probabilities proportional to $0.35^2$ and $0.25^2$. These are $49/74$ for \texttt{he is} and $25/74$ for \texttt{she is}. This establishes the example at the boundary, which the strict inequality in Proposition~\ref{prop:compatible-mixtures} does not cover. The group values become equal at $\alpha_\star\approx0.6162$. Since $m\alpha_\star>1$, either single-group predictor is then optimal, but mixtures between the groups are not.

\paragraph{Why the support condition matters.}
Consider instead two binary tokens with joint probabilities
\begin{equation*}
 q=\begin{pmatrix}0.4&0.3\\0.3&0\end{pmatrix},
 \qquad\text{rows and columns indexed by }0,1.
\end{equation*}
The pairs $00$, $01$, and $10$ are valid, while $11$ is not. These valid pairs cannot be divided into the disjoint Cartesian groups required by the proposition.

For every $0<\alpha<1$, an optimal factorized predictor assigns positive probability to all four pairs. To see this, use Step~4 of the proof above, which applies to any target. Token $0$ has a positive coefficient at either position regardless of the other position's distribution, because it forms a valid pair with both $0$ and $1$. Thus $p_1(0)>0$ and $p_2(0)>0$. Token $1$ then also has a positive coefficient at each position, since it forms a valid pair with the other position's $0$. Hence $p_1(1)>0$ and $p_2(1)>0$, giving positive probability to the invalid pair $11$.

Factorization and increasing alpha alone therefore do not guarantee valid samples for $0<\alpha<1$. The support condition is essential. Proposition~\ref{prop:compatible-mixtures} concerns the population optimum at a fixed context, not convergence of neural-network training.
\endgroup

\section{\revision{TinyGSM comparison protocols}}
\label{app:tinygsm-protocols}
{
\reviewrevision{We document the TinyGSM training and evaluation protocols and the sources of baseline results. Appendix~\ref{app:tinygsm-adaptive} then extends the fixed-budget comparisons to adaptive confidence sampling.}

\paragraph{AlphaDLM \newrevision{(ours)}.}
The teaser uses one EMA checkpoint at 250k updates, with $k=16$ after $k=1$ pretraining to 200k. Evaluation uses 512-token sequences, the SmolLM-135M tokenizer, exact argmax token selection, and untempered confidence scores. Each round reveals the union of positions above confidence 0.999 and the highest-confidence $\lceil m/r\rceil$ remaining positions, where $m$ is the current mask count and $r$ is the number of rounds left. Revealed tokens remain fixed. \presentationrevision{This fixed-NFE confidence sampler} adapts the commitment rule of FMLM+\presentationrevision{~\citep{agarwal2026posterior}} to an absorbing-mask denoiser with a fixed number of model evaluations.

Every example executes all $R$ forwards, including rounds after all masks have been filled. Accuracy is the observed correct fraction over 1,319 test problems. These are single-checkpoint estimates, without repeated-training uncertainty.

\paragraph{Program verification.}
The shared GSM8K verifier extracts the first fenced code block when present, discards text before the first function definition, and attempts to remove trailing text that prevents parsing. It executes \texttt{simple\_math\_problem()} with a 5\,s time limit and reads its numeric return value or, when unavailable, a numeric answer from its printed output. This value is compared with the number following \texttt{\#\#\#\#} in the reference answer. Integers must match exactly, while floating-point comparisons use absolute tolerance $10^{-3}$. Missing functions, parsing or execution failures, timeouts, and missing numeric answers count as incorrect.

\paragraph{Common-initialization objective comparisons.}
The $k=1$ checkpoint at 200k updates initializes the sequence-level and token-wise continuations, which are evaluated at 250k. Table~\ref{tab:tinygsm-training} specifies the shared training recipe. Sequence-level and token-wise objectives first normalize over the active positions of each example and then average over examples. Neither objective uses an additional $1/t$ weight. Examples with no active positions contribute zero.

We tokenize each TinyGSM example as a beginning-of-sequence token, the question, the literal two-character separator \texttt{\textbackslash n}, the Python solution, and an end-of-sequence token. We exclude examples longer than 512 tokens and pad the remaining examples to this length without packing different examples together. The question and separator remain clean context. Response, end-of-sequence, and padding positions participate in corruption and training. A fixed 1\% split is reserved for validation with split seed 42. The tokenizer is \texttt{HuggingFaceTB/SmolLM-135M}.

Training times are stratified using the implementation's antithetic sampler for a uniform distribution on $[\varepsilon,1]$, with $\varepsilon=10^{-3}$. \newrevision{The probability of keeping a token unchanged follows the stabilized linear schedule} $\bar\alpha_t=\varepsilon+(1-\varepsilon)(1-t)$. It uses no adaptive schedule or importance sampling. The model does not receive a time-conditioning signal.

\begin{table}[ht]
\centering
\caption{\textbf{Shared TinyGSM training recipe.} Sequence-level and token-wise continuations use the same architecture, data processing, and optimizer settings.}
\label{tab:tinygsm-training}
{\small
\renewcommand{\arraystretch}{1.1}
\begin{tabular}{@{}>{\raggedright\arraybackslash}p{0.31\linewidth}>{\raggedright\arraybackslash}p{0.66\linewidth}@{}}
\toprule
Component & Setting \\
\midrule
Denoiser & Bidirectional DiT with 12 blocks, 12 attention heads, hidden width 768, and conditioning width 128 \\
Regularization & Dropout 0.1 and zero weight decay \\
Sequence length & 512 tokens including prompt, response, and padding \\
Optimizer & AdamW, learning rate $3\times10^{-4}$, $(\beta_1,\beta_2)=(0.9,0.999)$, $\epsilon=10^{-8}$ \\
Learning-rate schedule & 2,500 warmup updates followed by a constant rate \\
Optimization & Global batch size 512, gradient clipping 1.0, and training seed 1 \\
Weight averaging & EMA decay 0.9999 \\
Training duration & Common sequence-level $k=1$ prefix through 200k updates, followed by 50k continuation updates \\
\bottomrule
\end{tabular}
}
\end{table}

The one-evaluation token-wise control uses EMA weights, exact argmax tokens, and confidence threshold zero, revealing all missing tokens in a single forward. In the 32-evaluation ancestral comparison, $T=0$ takes the argmax over clean-token predictions \newrevision{and still samples reveal positions from the ancestral transition}. It is not an argmax over the complete reverse transition including the mask state.

\begin{table}[ht]
\centering
\caption{\revision{GSM8K accuracy (\%) \reviewrevision{$\uparrow$} with a shared initialization and \presentationrevision{fixed-NFE confidence sampling}. All checkpoints use 250k total updates, including the common $k=1$ prefix through 200k. Each value is an observed correct fraction over 1,319 examples.}}
\label{tab:tinygsm-matched-budget}
{\small
\begin{tabular}{lrrrrr}
\toprule
Objective & 1 NFE & 2 NFE & 4 NFE & 8 NFE & 16 NFE \\
\midrule
Sequence $k=1$ & 0.91 & 0.91 & 8.11 & 26.23 & 40.56 \\
Sequence $k=2$ & 0.99 & 1.29 & 10.61 & 28.58 & 43.06 \\
Sequence $k=4$ & 1.82 & 2.58 & 16.68 & 35.33 & 46.47 \\
Sequence $k=8$ & 2.73 & 7.73 & 28.89 & 43.97 & 49.20 \\
Sequence $k=16$ & 6.29 & 18.04 & 34.65 & 39.04 & 40.56 \\
Token $\alpha=0.25$ & 0.83 & 1.74 & 12.66 & 33.28 & 44.35 \\
Token $\alpha=0.375$ & 1.67 & 3.56 & 17.97 & 35.63 & 44.12 \\
Token $\alpha=0.5$ & 1.36 & 5.46 & 19.86 & 36.01 & 39.95 \\
Token $\alpha=0.75$ & 2.12 & 5.23 & 7.66 & 7.88 & 7.88 \\
\bottomrule
\end{tabular}

}
\end{table}

The fixed-budget objective comparison evaluates each checkpoint at 1, 2, 4, 8, and 16 forwards using the same sampling protocol. The confidence threshold is 0.999, tokens are exact argmax predictions, and all models use four GPUs with batch size 16 per GPU, EMA weights, and seed 1. Every example executes the prescribed number of forwards. Table~\ref{tab:tinygsm-matched-budget} reports all points.

\paragraph{\presentationrevision{MDLM with fixed-NFE confidence sampling.}}
We evaluate the released MDLM checkpoint from the $\mathbb{S}$-FLM repository on the same 1,319 problems at $R\in\{1,2,4,8,16,32,64\}$. We use the same \presentationrevision{fixed-NFE confidence sampling} implementation and evaluation configurations as \newrevision{our method}, including the tokenizer, confidence threshold, argmax token selection, and EMA weights. Each example's recorded forward count is verified against $R$. The respective correct counts are 10, 9, 56, 265, 497, 579, and 536. This controls the \newrevision{sampler}, while the training objectives and initialization histories differ.

\paragraph{Reevaluation of released checkpoints.}
We evaluate the released TinyGSM MDLM\presentationrevision{~\citep{sahoo2024mdlm}}, DUO\presentationrevision{~\citep{sahoo2025duo}}, and $\mathbb{S}$-FLM\presentationrevision{~\citep{deschenaux2026sflm}} checkpoints with the authors' sampling implementation. MDLM and DUO use ancestral sampling at $T=0.1$. Both $\mathbb{S}$-FLM curves use the spherical backbone with the trained adaptive truncated schedule, with exact velocity at $T=0.1$ or top-1 velocity at $T=1$, followed by greedy final decoding. We use EMA weights, seed 1, four GPUs, and batch size 32 per GPU. The plot and table report integer correct counts divided by 1,319, rounded to one decimal in the table.

\paragraph{Published table values.}

FMLM+ uses greedy posterior refinement with threshold 0.999 from Table~8 of \citet{agarwal2026posterior}. We show Init in the teaser and all three variants in Table~\ref{tab:tinygsm-published}. DBTM uses greedy refinement with threshold 0.8 from Table~8 of \citet{tang2026dbtm}. \presentationrevision{For IDLM, we use the published values from Table~1 of \citet{li2026idlm} at 32 and 64 NFE. We evaluate the released IDLM--MDLM and IDLM--DUO TinyGSM checkpoints at 4, 8, and 16 NFE using the $\mathbb{S}$-FLM ancestral sampler at $T=1$, EMA weights, seed 1, and batch size 32 on one GPU. We verify the exact forward count for all 1,319 examples. The correct counts are 5, 24, and 86 for IDLM--MDLM and 30, 93, and 154 for IDLM--DUO. As a reproduction check, our 32-NFE evaluations yield 12.51\% and 17.06\%, compared with the published 12.81\% and 15.39\%.} These methods \newrevision{use} their published training recipes. FMLM+ Init and the distillation variants incur teacher training in addition to their own training, so this comparison does not match total training compute.

For context, the autoregressive references in \citet{deschenaux2026sflm} reach 53.9\% with sampling and 63.3\% with greedy decoding. 

}

\noindent\begin{minipage}{\linewidth}
\centering
\captionof{table}{\textbf{Ancestral accuracy (\%) \reviewrevision{$\uparrow$} at 32 NFE.} Checkpoints share the $k=1$ initialization and 250k total updates. At $T=0$, clean-token predictions are argmax while reveal positions remain stochastic.}
\label{tab:tinygsm-fixed}
\begin{tabular}{lrr}
\toprule
Training parameter & Ancestral ($T=1$) & Ancestral ($T=0$) \\
\midrule
k=1 & 8.19 & 26.00 \\
k=2 & 10.46 & 26.61 \\
k=4 & 15.09 & 29.80 \\
k=8 & 21.30 & 34.95 \\
k=16 & 32.52 & 37.07 \\
\bottomrule
\end{tabular}

\end{minipage}

\subsection{\presentationrevision{Adaptive confidence sampling}}
\label{app:tinygsm-adaptive}
\presentationrevision{Here we use adaptive confidence sampling from Fast-dLLM~\citep{wu2025fastdllm} (Section~\ref{sec:parallel-decoding}) with argmax tokens and no prescribed NFE. We vary the confidence threshold, so each point reports the resulting average NFE (Figure~\ref{fig:tinygsm}, Table~\ref{tab:tinygsm-confidence}).} At threshold 0.9, $k=8$ reaches 49.43\% at 17.83 NFE, while $k=1$ reaches 43.06\% at 57.59 NFE. The $k=16$ run attains 40.03\% at 5.70 NFE. Token-wise $\alpha=0.25$ reaches 47.31\% at 32.00 NFE, so the sequence-level $k=8$ point has 2.12 percentage points higher accuracy with 1.79 times fewer evaluations.

\noindent\begin{minipage}{\linewidth}
\centering
\includegraphics[width=\linewidth]{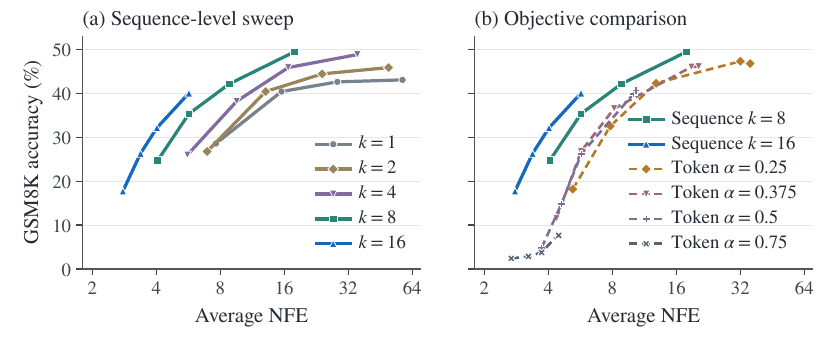}
\captionof{figure}{\textbf{Accuracy--NFE trade-offs under \presentationrevision{adaptive confidence sampling}.} Left: AlphaDLM (ours) parameter sweep. Right: sequence-level and token-wise objectives. All checkpoints share the $k=1$ initialization and 250k total updates. Points vary the confidence threshold. Higher accuracy and fewer evaluations are preferable.}
\label{fig:tinygsm}
\end{minipage}

\noindent\begin{minipage}{\linewidth}
\centering
\captionof{table}{\textbf{\presentationrevision{Adaptive confidence sampling} from a common initialization.} Each entry gives accuracy (\%) \reviewrevision{$\uparrow$} / average NFE \reviewrevision{$\downarrow$} on all 1,319 examples. The checkpoints are evaluated at 250k updates after the common sequence-level $k=1$ initialization at 200k.}
\label{tab:tinygsm-confidence}
{\small
\begin{tabular}{lrrrr}
\toprule
\presentationrevision{Objective} & $\tau=0.6$ & $\tau=0.7$ & $\tau=0.8$ & $\tau=0.9$ \\
\midrule
\presentationrevision{Sequence $k=1$} & 28.58 / 7.64 & 40.41 / 15.54 & 42.61 / 28.42 & 43.06 / 57.59 \\
\presentationrevision{Sequence $k=2$} & 26.76 / 6.94 & 40.41 / 13.09 & 44.43 / 24.12 & 45.87 / 49.54 \\
\presentationrevision{Sequence $k=4$} & 26.00 / 5.61 & 38.21 / 9.57 & 45.87 / 16.72 & 48.82 / 35.32 \\
\presentationrevision{Sequence $k=8$} & 24.72 / 4.06 & 35.33 / 5.69 & 42.15 / 8.83 & 49.43 / 17.83 \\
\presentationrevision{Sequence $k=16$} & 17.74 / 2.79 & 26.31 / 3.37 & 32.15 / 4.03 & 40.03 / 5.70 \\
\midrule
\presentationrevision{Token $\alpha=0.25$} & \presentationrevision{18.20 / 5.21} & \presentationrevision{32.52 / 7.82} & \presentationrevision{42.30 / 12.86} & \presentationrevision{47.31 / 32.00} \\
\presentationrevision{Token $\alpha=0.375$} & \presentationrevision{11.52 / 4.37} & \presentationrevision{26.76 / 5.70} & \presentationrevision{36.54 / 8.15} & \presentationrevision{46.02 / 18.84} \\
\presentationrevision{Token $\alpha=0.5$} & \presentationrevision{4.85 / 3.72} & \presentationrevision{14.78 / 4.60} & \presentationrevision{26.23 / 5.72} & \presentationrevision{39.88 / 10.06} \\
\presentationrevision{Token $\alpha=0.75$} & \presentationrevision{2.43 / 2.68} & \presentationrevision{2.88 / 3.23} & \presentationrevision{3.79 / 3.72} & \presentationrevision{7.66 / 4.48} \\
\bottomrule
\end{tabular}}
\end{minipage}

\section{SDAR experimental details}
\label{app:sdar-protocols}

\paragraph{Data and training.}
We fine-tune \texttt{JetLM/SDAR-1.7B-Chat} on the same data for all objectives within each domain.
For mathematics, we use 272,872 examples from the CoT subset of Nemotron-SFT-Math-v4.\footnote{\url{https://huggingface.co/datasets/nvidia/Nemotron-SFT-Math-v4}.}
We \newrevision{use} assistant content, omit separate reasoning fields, remove lexical overlap candidates with GSM8K and MATH500, and keep complete examples with at most 256 prompt and 1,280 response tokens.
For code, we use OpenCodeInstruct~\citep{ahmad2025opencodeinstruct}.
For all setups, we use a batch size of 128 and a learning rate of $5 \times 10^{-6}$, and train for 4,000 steps.
CE+CAP~\citep{bie2025llada2} uses a correctness-conditioned entropy regularizer of weight $0.5$ to encourage confident correct predictions.

{
\paragraph{Block-level objective.}
\newrevision{Our method applies} Equation~\ref{eq:effective-alpha} separately to each four-token response block. For the masked supervised positions $M_b$ in block $b$, we use
\begin{equation}
 s_b=\frac{1}{|M_b|}\sum_{j\in M_b}\log p_{\theta,j}(y_j\mid\mathrm{context}),
 \qquad \ell_b(k)=\frac{1-\exp(k s_b)}{k}.
\end{equation}
The $k=0$ limit is mean cross-entropy within the block. A block with no active targets contributes zero. We sum block losses within each example, divide by a fixed dataset estimate of the mean number of valid response blocks, and average over examples. Thus, the nonlinear transformation couples targets within a block rather than across the entire response.

\paragraph{Confidence-aware parallel training.}
CAP adds an entropy penalty on correctly predicted masked targets. Let $C$ contain masked supervised positions whose logit argmax equals the target and let $r_j=\operatorname{softmax}(z_j/0.5)$, where $z_j$ denotes the vocabulary logits. The objective is
\begin{equation}
 \mathcal L_{\mathrm{CE+CAP}}=\mathcal L_{\mathrm{CE}}+
 \frac{0.5}{|C|}\sum_{j\in C}H(r_j).
\end{equation}
The entropy term is zero when $C$ is empty. The correctness gate is discrete, while gradients pass through the entropy of the selected distributions. The normalization uses the correct masked positions in each microbatch.

\par}

\par\medskip\paragraph{Evaluation.}
We evaluate accuracy on 1,319 GSM8K problems~\citep{cobbe2021verifiers} and 500 MATH500 problems~\citep{hendrycks2021math}, and pass@1 on 164 HumanEval-Instruct problems~\citep{chen2021evaluating} and 427 MBPP-Instruct problems~\citep{austin2021program}.
\presentationrevision{We use adaptive confidence sampling within four-token blocks}, with confidence thresholds $\tau\in\{0.75,0.80,0.85,0.90,0.95,0.97,0.99\}$ for mathematics and $\tau\in\{0.60,0.70,0.80,0.85,0.90,0.95,0.97,0.99\}$ for code.

{
Candidate tokens are selected by argmax ($T=0$). At each iteration, the \newrevision{sampler} accepts remaining positions whose confidence is strictly above the threshold. If none qualifies, it accepts the highest-confidence position. Accepted tokens are not remasked. Each block therefore requires at most four denoising evaluations. Generation uses prefix KV caching and response limits of 1,280 tokens for mathematics and 384 for code, stopping after completing the block containing EOS or the turn-end token.

\par}

\noindent\begin{minipage}{\linewidth}
\centering
\includegraphics[width=0.85\linewidth]{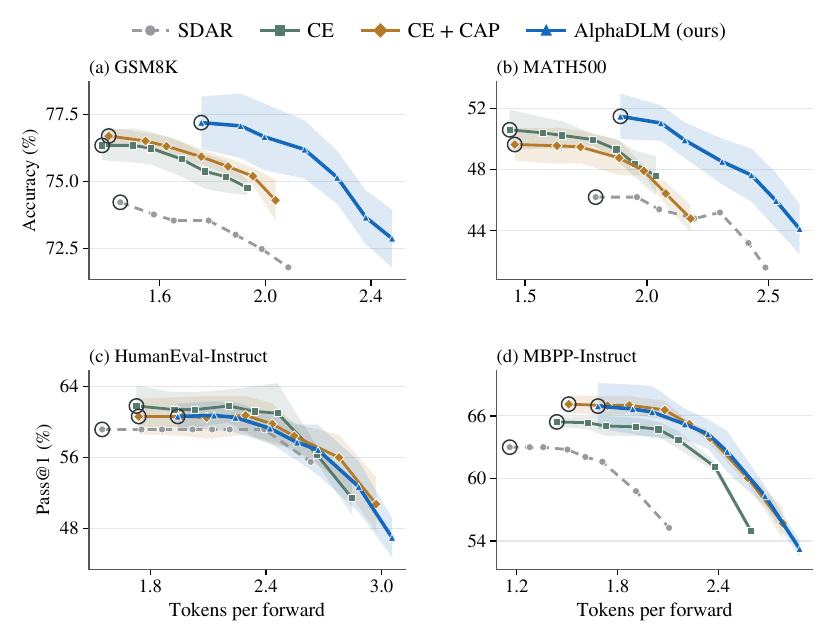}
\captionof{figure}{\textbf{SDAR accuracy--TPF trade-offs.} AlphaDLM \newrevision{(ours)} uses $k=0.8$ for mathematics and $k=0.4$ for code. Points show means at each confidence threshold, with bands of $\pm1$ sample standard deviation in accuracy across runs. \presentationrevision{Rings mark threshold 0.99. Axes match Figure~\ref{fig:sdar-math}~(a--d).}}
\label{fig:sdar-uncertainty}
\end{minipage}\par\medskip

\paragraph{Tokens per forward (TPF).}
For a benchmark with $N$ examples, let $L_i$ be the number of returned response tokens and $D_i$ the number of denoising forwards used for example $i$.
We compute
\begin{equation}
    \mathrm{TPF} = \frac{\sum_{i=1}^{N} L_i}{\sum_{i=1}^{N}(D_i+1)},
\end{equation}
where the additional forward accounts for prompt prefill once per example.
The numerator excludes prompt tokens, the first EOS or stop token, and all subsequent tokens, and respects the response-length limit.
Denoising forwards are counted separately for each active example, including all steps in the block containing its stop token. Additional forwards used only to update the KV cache are excluded.
We aggregate token and forward counts over the benchmark before taking their ratio, then average the resulting TPF values across runs at each confidence threshold.
With four-token blocks, TPF is at most four. Prompt prefill and partially filled final blocks reduce the measured value even when each block requires only one denoising step.

\paragraph{Aggregation.}
At each threshold, we average accuracy and TPF separately across 5 runs.
Figure~\ref{fig:sdar-uncertainty} shows the main-text means with $\pm1$ sample standard deviation in accuracy.

\paragraph{Parameter sweep.}
Figure~\ref{fig:sdar-ablation-uncertainty} compares $k\in\{0,0.4,0.8,1.2,1.6,2.0\}$ on mathematics with a learning rate of $5\times10^{-6}$ and a batch size of 128, using 4,000 training steps.

\par\medskip\noindent\begin{minipage}{\linewidth}
\centering
\includegraphics[width=\linewidth]{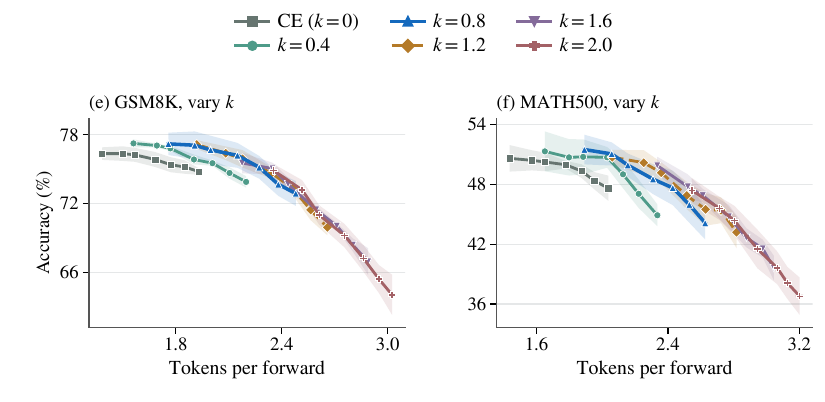}
\captionof{figure}{\textbf{SDAR parameter sweep for AlphaDLM (ours) and CE.} \presentationrevision{Panels (e,f) use the same axes as Figure~\ref{fig:sdar-ablation}~(e,f).} Bands show $\pm1$ sample standard deviation in accuracy.}
\label{fig:sdar-ablation-uncertainty}
\end{minipage}\par\medskip

\section{Sensitivity to training initialization}
\label{app:tinygsm-initialization}
\reviewrevision{The main TinyGSM comparison starts from a sequence-level $k=1$ checkpoint at 200k updates. Here we repeat the objective comparison from a CE-trained checkpoint at the same update count. Both sets of continuations train for another 50k updates. We compare sequence-level $k\in\{2,4,8,16\}$ and token-wise $\alpha\in\{0.25,0.5,0.75\}$, using the same architecture, optimizer settings, and \presentationrevision{fixed-NFE confidence sampler}. Each evaluation uses EMA weights, exact argmax token selection, confidence threshold 0.999, and exactly 1, 2, 4, 8, or 16 forwards on all 1,319 GSM8K problems. There is one training seed per recipe.}

\begin{figure}[!htb]
\centering
\includegraphics[width=\linewidth]{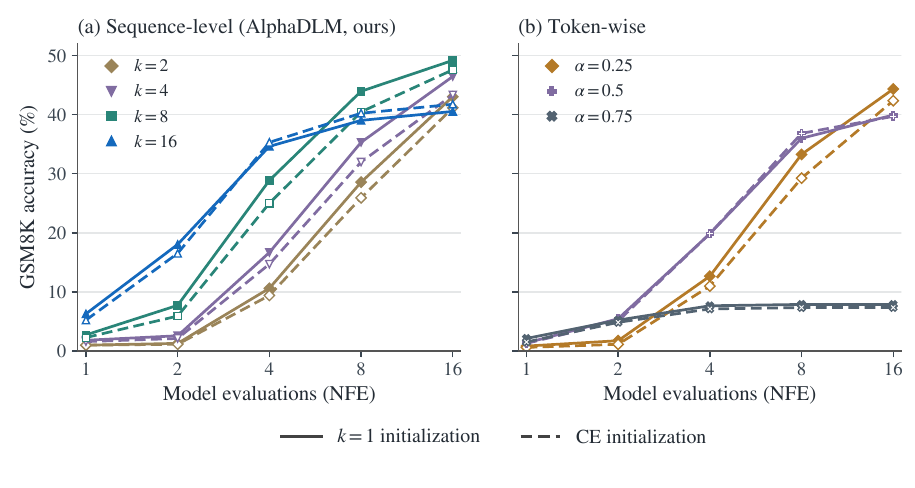}
\caption{\reviewrevision{\textbf{Objective comparisons under two initializations.} Colors identify the continuation objective. Solid lines start from $k=1$ and dashed lines from CE, both at 200k updates and evaluated at 250k. Only parameters available under both initializations are shown.}}
\label{fig:tinygsm-initialization}
\end{figure}

\reviewrevision{The qualitative advantage of sequence-level training persists under CE initialization (Figure~\ref{fig:tinygsm-initialization}). At four NFE, sequence-level $k=16$ reaches 35.33\%, compared with 19.86\% for the strongest tested token-wise setting. At sixteen NFE, $k=8$ reaches 47.61\%, compared with 42.38\% for the strongest tested token-wise setting. Among the tested sequence-level parameters, $k=16$ is best at 1--4 NFE and $k=8$ at 8--16 NFE under both initializations.}

\reviewrevision{The $k=1$ initialization generally yields higher accuracy for subsequent sequence-level $k=2,4,8$ training. However, CE initialization improves $k=16$ at 4--16 NFE, by 0.68--1.22 percentage points. Token-wise differences also depend on the loss parameter and evaluation budget (Table~\ref{tab:tinygsm-initialization}). These single-seed comparisons support robustness of the objective comparison, rather than a uniform advantage of either initialization.}

\begin{table}[!htb]
\centering
\caption{\reviewrevision{\textbf{GSM8K accuracy (\%) $\uparrow$ under two initializations.} Each paired row uses the same continuation objective and decoding protocol. The final row reports the CE training run continued to 250k, whose 200k checkpoint initializes the CE-start experiments. It is a reference, not a matched-objective comparison with the $k=1$ continuation.}}
\label{tab:tinygsm-initialization}
{\small\setlength{\tabcolsep}{5pt}\begin{tabular}{llrrrrr}
\toprule
Objective & Initialization & 1 NFE & 2 & 4 & 8 & 16 \\
\midrule
Sequence $k=2$ & $k=1$ & 0.99 & 1.29 & 10.61 & 28.58 & 43.06 \\
Sequence $k=2$ & CE & 0.99 & 1.14 & 9.40 & 25.93 & 41.17 \\
\addlinespace[3pt]
Sequence $k=4$ & $k=1$ & 1.82 & 2.58 & 16.68 & 35.33 & 46.47 \\
Sequence $k=4$ & CE & 1.59 & 2.12 & 14.71 & 31.99 & 43.37 \\
\addlinespace[3pt]
Sequence $k=8$ & $k=1$ & 2.73 & 7.73 & 28.89 & 43.97 & 49.20 \\
Sequence $k=8$ & CE & 2.27 & 5.91 & 25.02 & 40.49 & 47.61 \\
\addlinespace[3pt]
Sequence $k=16$ & $k=1$ & 6.29 & 18.04 & 34.65 & 39.04 & 40.56 \\
Sequence $k=16$ & CE & 5.23 & 16.53 & 35.33 & 40.26 & 41.77 \\
\midrule
Token $\alpha=0.25$ & $k=1$ & 0.83 & 1.74 & 12.66 & 33.28 & 44.35 \\
Token $\alpha=0.25$ & CE & 0.61 & 1.14 & 10.99 & 29.26 & 42.38 \\
\addlinespace[3pt]
Token $\alpha=0.5$ & $k=1$ & 1.36 & 5.46 & 19.86 & 36.01 & 39.95 \\
Token $\alpha=0.5$ & CE & 1.29 & 5.08 & 19.86 & 36.85 & 39.73 \\
\addlinespace[3pt]
Token $\alpha=0.75$ & $k=1$ & 2.12 & 5.23 & 7.66 & 7.88 & 7.88 \\
Token $\alpha=0.75$ & CE & 1.52 & 4.85 & 7.13 & 7.35 & 7.35 \\
\midrule
CE & CE & 0.45 & 0.53 & 5.69 & 18.50 & 36.32 \\
\bottomrule
\end{tabular}
}
\end{table}

\reviewrevision{This ablation covers fixed-budget decoding and the shared parameter grid only. The $k=1$-start results for token-wise $\alpha=0.375$, sequence-level $k=32,64$, ancestral decoding, and pass@$K$ remain separate experiments.}

\raggedbottom
\Needspace{20\baselineskip}
\section{Multiple-answer generation}
\label{app:passk}

We study how the training objective affects the benefit of generating several answers to the same problem.
We use the three sequence-level checkpoints with $k\in\{1,8,16\}$ at 250k training updates, all continued from the common $k=1$ checkpoint at 200k updates.
For each of the 1,319 GSM8K test problems, we generate $n=32$ answers with independently seeded ancestral sampling.
We consider $R\in\{4,16\}$ model evaluations per answer and temperatures $T\in\{0,1\}$, giving 12 configurations and 506,496 generated answers in total.
At $T=0$, the clean-token prediction is the argmax, while the ancestral reveal trajectory remains random.
All runs use EMA weights, top-$p=1$, the 512-token sequence limit, and the same program verifier as the single-answer evaluation.

Let $c_i$ denote the number of correct answers among the 32 samples for problem $i$.
We estimate the probability that at least one of $K$ answers is correct using the unbiased pass@$K$ estimator of \citet{chen2021evaluating},
\begin{equation}
\widehat{\mathrm{pass@}K}
=\frac{1}{1319}\sum_{i=1}^{1319}
\left(1-\frac{\binom{32-c_i}{K}}{\binom{32}{K}}\right),
\qquad K\in\{1,2,4,8,16,32\},
\label{eq:passk_estimator}
\end{equation}
where the numerator is zero when $K>32-c_i$.
\newrevision{We include repeated answers in the estimator.}
This metric measures verifier-assisted coverage among $K$ attempts, with a total generation budget of $KR$ model evaluations per problem.
The sampling repeats use one trained checkpoint for each value of $k$.

Figure~\ref{fig:tinygsm_passk} and Table~\ref{tab:tinygsm_passk} show that the preferred training parameter depends on the number of attempts.
With $T=0$ and $R=16$, increasing $k$ from 1 to 16 improves pass@1 from 16.83\% to 32.24\%.
At 32 attempts, $k=8$ reaches 74.22\%, compared with 72.71\% for $k=16$ and 68.39\% for $k=1$.
At $T=1$, $k=16$ has the highest observed pass@$K$ across the tested values of $K$ at both per-answer budgets.
\par\medskip\noindent\begin{minipage}{\linewidth}
  \centering
  \setlength{\abovecaptionskip}{5pt}
  \includegraphics[width=0.86\linewidth,trim=0 8bp 0 0,clip]{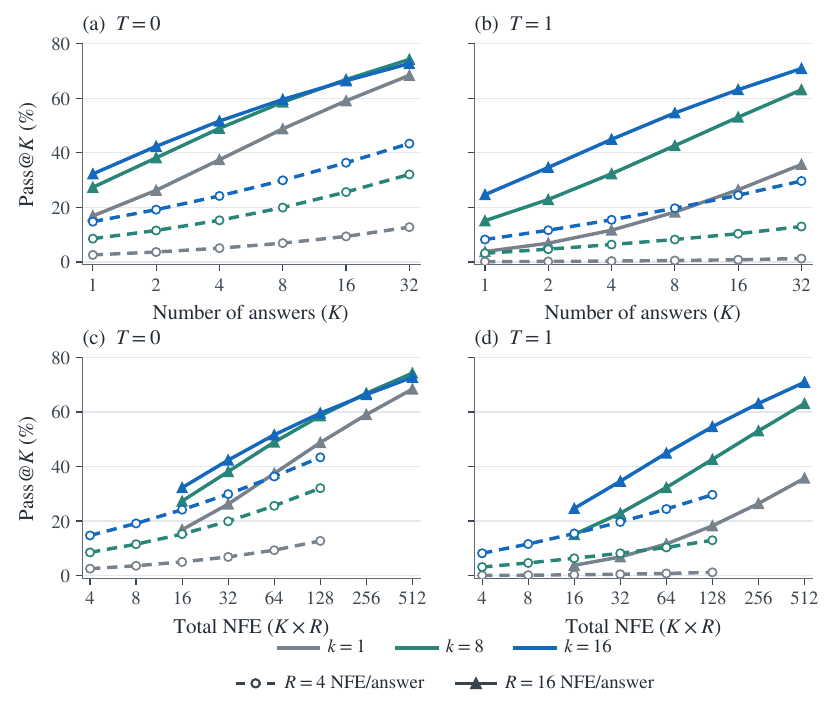}
  \captionof{figure}{\presentationrevision{\textbf{GSM8K pass@$K$ with ancestral sampling.} Top: coverage versus number of answers. Bottom: the same results versus total NFE, $KR$. Colors identify $k$, and line styles identify NFE per answer $R$. Estimates use 32 attempts per problem. Reveal positions remain random at $T=0$.}}
  \label{fig:tinygsm_passk}
\end{minipage}\par\medskip

\presentationrevision{At each shared total budget (16, 32, 64, and 128 NFE), using sixteen evaluations per answer gives higher coverage than using four for every tested checkpoint and temperature.}

\begin{table}[!htb]
  \centering
  \small
  \setlength{\tabcolsep}{4pt}
  \caption{GSM8K pass@$K$ (\%) \reviewrevision{$\uparrow$} estimated from 32 answers per problem.
  $T$ is the sampling temperature, $R$ is NFE per answer, and $k$ is the training parameter.
  Bold entries mark the highest value across training parameters for the same $T$, $R$, and $K$.}
  \label{tab:tinygsm_passk}
  {\centering\begin{tabular}{rrrrrrrrr}
\toprule
$T$ & $R$ & $k$ & Pass@1\reviewrevision{$\uparrow$} & Pass@2\reviewrevision{$\uparrow$} & Pass@4\reviewrevision{$\uparrow$} & Pass@8\reviewrevision{$\uparrow$} & Pass@16\reviewrevision{$\uparrow$} & Pass@32\reviewrevision{$\uparrow$} \\
\midrule
0 & 4 & 1 & 2.56 & 3.62 & 5.03 & 6.86 & 9.33 & 12.74 \\
0 & 4 & 8 & 8.51 & 11.52 & 15.23 & 19.91 & 25.60 & 32.07 \\
0 & 4 & 16 & \textbf{14.73} & \textbf{19.14} & \textbf{24.17} & \textbf{29.89} & \textbf{36.34} & \textbf{43.37} \\
\midrule
0 & 16 & 1 & 16.83 & 26.23 & 37.45 & 48.77 & 59.04 & 68.39 \\
0 & 16 & 8 & 27.33 & 38.14 & 48.95 & 58.56 & \textbf{66.85} & \textbf{74.22} \\
0 & 16 & 16 & \textbf{32.24} & \textbf{42.35} & \textbf{51.62} & \textbf{59.55} & 66.39 & 72.71 \\
\midrule
1 & 4 & 1 & 0.09 & 0.17 & 0.30 & 0.50 & 0.80 & 1.21 \\
1 & 4 & 8 & 3.16 & 4.65 & 6.36 & 8.20 & 10.33 & 12.96 \\
1 & 4 & 16 & \textbf{8.20} & \textbf{11.59} & \textbf{15.42} & \textbf{19.67} & \textbf{24.42} & \textbf{29.64} \\
\midrule
1 & 16 & 1 & 3.79 & 6.84 & 11.62 & 18.25 & 26.43 & 35.71 \\
1 & 16 & 8 & 15.08 & 22.87 & 32.31 & 42.62 & 53.04 & 63.08 \\
1 & 16 & 16 & \textbf{24.64} & \textbf{34.61} & \textbf{44.91} & \textbf{54.59} & \textbf{63.12} & \textbf{70.89} \\
\bottomrule
\end{tabular}
\par}
\end{table}

\end{document}